\documentclass{article} 
\usepackage{jmlr2e}
\usepackage{amsmath}
\usepackage[normalem]{ulem}
\usepackage{algorithm}
\usepackage{color}
\usepackage{algorithmic}
\usepackage{rotating}
\usepackage{changepage}
\usepackage{verbatim}
\usepackage{cleveref}

\newcommand{\bw}{{\bf w}}
\newcommand{\bx}{{\bf x}}
\newcommand{\by}{{\by y}}

\newcommand{\be}{\begin{eqnarray}}
\newcommand{\ben}{\begin{eqnarray*}}
\newcommand{\en}{\end{eqnarray}}
\newcommand{\enn}{\end{eqnarray*}}

\newcommand{\half}{\frac{1}{2}}
\graphicspath{{fig/}}

\begin{document}

\title{ Online AUC Optimization Based on Second-order Surrogate Loss}

\author{\name Junru Luo\email luojunru@cczu.edu.cn \\
       \addr Aliyun School of Big Data\\
       Changzhou University\\
       Changzhou,Jiangsu 213164, China
       \AND
       \name Difei Cheng\email chengdifei@amss.ac.cn\\
       \addr School of Computer Science and Artificial Intelligence\\
       Aerospace Information Technology University\\
       Jinan, Shandong 250299, China
       \AND
       \name Bo Zhang\email b.zhang@amt.ac.cn\\
       \addr Academy of Mathematics and Systems Science\\
       Chinese Academy of Sciences\\
       Beijing 100190, China
       \AND
       \addr School of Mathematical Sciences\\
       University of Chinese Academy of Sciences\\
       Beijing 100049, China}

\editor{Noname}

\maketitle

\begin{abstract}
The Area Under the Curve (AUC) is an important performance metric for classification tasks, particularly in class-imbalanced scenarios. However, optimizing the AUC presents significant challenges due to the non-convex and discontinuous nature of pairwise 0/1 losses, which are difficult to optimize, as well as the substantial memory cost of instance-wise storage, which creates bottlenecks in large-scale applications. 
To overcome these challenges, we propose a novel second-order surrogate loss based on the pairwise hinge loss, and develop an efficient online algorithm. Unlike conventional approaches that approximate each individual pairwise 0/1 loss term with an instance-wise surrogate function, our approach introduces a new paradigm that directly substitutes the entire aggregated pairwise loss with a surrogate loss function constructed from the first- and second-order statistics of the training data.
Theoretically, while existing online AUC optimization algorithms typically achieve an $\mathcal{O}(\sqrt{T})$ regret bound, our method attains a tighter $\mathcal{O}(\ln T)$ bound. Furthermore, we extend the proposed framework to nonlinear settings through a kernel-based formulation. Extensive experiments on multiple benchmark datasets demonstrate the superior efficiency and effectiveness of the proposed second-order surrogate loss in optimizing online AUC performance.

\end{abstract}
\begin{keywords}
AUC maximization, online learning, surrogate loss function, robust learning, kernel method
\end{keywords}

\section{Introduction}
While accuracy remains a fundamental evaluation metric for classification tasks, its utility diminishes significantly in class-imbalanced scenarios where naive majority classifiers can achieve misleadingly high scores \citep{He2009,Johnson2019}. This limitation has motivated the development of alternative performance measures better suited to imbalanced classification, including precision-recall curves, F-scores, and most notably, the Area Under the ROC Curve (AUC) \citep{Sokolova2009,Juba2019,Luo2024,Christen2023,Agarwal2013}. Rooted in the Wilcoxon-Mann-Whitney statistic, the AUC metric quantifies the probability that a randomly selected positive instance is ranked higher than a randomly selected negative one \citep{Hanle1982,Cortes2004}. Recent years have witnessed growing interest in AUC optimization methods for handling class imbalance, accompanied by significant theoretical and algorithmic advances \citep{Lei2021,Yang2022,Zhu2023,Xie2024,Bao2025,Zhang2025}.

The minimization of AUC risk presents two fundamental challenges \citep{Calauzenes2013}. First, the original formulation involves minimizing a sum of pairwise 0/1 losses, which is hard to optimize due to the non-convex and discontinuous nature of the indicator function. Standard approaches address this issue through convex surrogate losses such as the pairwise hinge or square loss \citep{Bartlett2006,Gao2015}. Second, the memory requirement grows linearly with the number of instances, leading to a space complexity of $\mathcal{O}(N P)$, which imposes substantial memory overhead in large-scale applications.

These computational challenges have spurred significant methodological innovations in recent years. Numerous batch learning methodologies address the AUC optimization problem by selectively discarding pairwise instances through various sampling or weighting strategies~\citep{Brefeld2005,Herschtal2004,Joachims2005}. In parallel, another line of research focuses on univariate loss minimization to circumvent the inherent complexity of pairwise learning. \citet{Kot2011} established that the regret of AUC-based scoring functions can be bounded via the regret of balanced variants of standard non-pairwise losses, such as the exponential and logistic losses. This theoretical foundation was subsequently extended by \citet{Agarwal2013} to a broader class of strongly proper loss functions, further generalizing the regret bounds. \citet{Lyu2018} introduced a novel surrogate loss, termed UBAUC, derived from a reformulation of the AUC risk that replaces pairwise comparisons with prediction rankings. This approach enables the estimation of an optimal scoring function through the minimization of a corresponding univariate loss, thereby enhancing computational efficiency while maintaining theoretical guarantees.

For large-scale applications, online AUC maximization has emerged as a particularly promising paradigm. Early buffer-based approaches, exemplified by the pioneering Online AUC Maximization (OAM) algorithm proposed by \citet{Zhao2011}, maintained separate buffers for positive and negative instances and approximated the AUC loss through comparisons with stored examples. Although effective, these methods incurred substantial memory overhead proportional to the buffer size. A critical breakthrough came with moment-based methods, particularly leveraging the pairwise square loss's unique property of admitting exact decomposition into statistical moments. This property was first exploited in the One-Pass AUC (OPAUC) algorithm, which enabled gradient computation using only mean and covariance statistics, reducing memory requirements to $O(p^2)$ while preserving convergence guarantees, where $p$ denotes the feature dimension \citep{Gao2016}.

The strategic use of the pairwise square loss has further advanced the development of efficient AUC maximization algorithms for high-dimensional imbalanced classification. A pivotal contribution in this direction was made by \citet{Ying2016a}, who reformulated the square-based AUC loss minimization as a stochastic saddle point problem (SPP) through an innovative primal-dual approach. This formulation, embodied in the Stochastic Online AUC Maximization (SOLAM) algorithm, fundamentally transformed the computational landscape by eliminating the need to store historical instances or their second-order covariance matrices. The key theoretical breakthrough of SOLAM lies in its achievement of optimal $\mathcal{O}(p)$ space and per-iteration time complexity while maintaining convergence guarantees, making it particularly suitable for high-dimensional streaming data scenarios. These advances have established crucial theoretical and computational foundations for extending AUC maximization to deep learning paradigms \citep{Zhang2023,Yuan2022}.

Despite these advances, the pairwise square loss—though enabling efficient moment-based optimization in methods such as OPAUC, SOLAM, and deep AUC learning—exhibits inherent limitations for classification tasks due to its regression-oriented nature\citep{hastie2009,Xu2024b}. The quadratic penalty imposes non-zero gradients on correctly classified examples and demonstrates sensitivity to outliers. In contrast, the hinge loss offers superior classification properties but traditionally requires instance-wise storage of $\mathcal{O}(Np)$, rendering it impractical for large-scale applications. This dichotomy raises a key question: Can we develop a statistically sufficient reformulation that preserves the hinge loss's classification advantages while avoiding instance storage?

This paper addresses the aforementioned issues, and the main contributions of this work are as follows:

\begin{enumerate}
    \item We propose a novel second-order surrogate loss, $\psi_M$, based on the pairwise hinge loss and constructed from the first- and second-order statistics of the training data. Our approach introduces a new paradigm that directly substitutes the entire aggregated pairwise loss with a surrogate loss function, in contrast to traditional methods that approximate each pairwise 0/1 loss term using an instance-wise surrogate function.

    \item We introduce OAUC-M, an online AUC maximization algorithm based on $\psi_M$, which achieves $\mathcal{O}(p^2)$ space complexity by using only statistical moments. This method achieves an $\mathcal{O}(\ln T)$ regret bound—the first for hinge-type AUC optimization—improving upon previous $\mathcal{O}(\sqrt{T})$ results.

    \item We extend our framework to nonlinear classification by developing OKAUC-M, an online kernelized AUC maximization algorithm based on $\psi_M$. This method effectively handles non-separable data in real-world tasks by leveraging reproducing kernel Hilbert spaces, while still maintaining sublinear regret bounds.
\end{enumerate}

The remainder of this paper is organized as follows. Section \ref{sec2} introduces the online AUC optimization problem. Section \ref{sec3} derives the second-order surrogate loss function $\psi_M$. Section \ref{sec4} presents the proposed online AUC maximization algorithm based on $\psi_M$. Section \ref{sec6} introduces a kernelized extension of the proposed method to address nonlinear classification tasks. Section \ref{sec7} evaluates the performance of various online AUC optimization algorithms experimentally. Finally, Section \ref{sec8} concludes the paper.

\section{Online AUC Optimization Problems}\label{sec2}

In imbalanced binary classification, the AUC is widely adopted as a performance metric for evaluating scoring functions $f \in \mathcal{F}$. It quantifies the probability that a randomly chosen positive instance $\bx$ receives a higher score than a randomly chosen negative instance $\bx'$:
\ben
AUC(f) = Pr\{ f(\bx) > f(\bx')| y = +1, y' = -1\}.
\enn
The Wilcoxon-Mann-Whitney statistic offers a nonparametric estimator of the AUC\citep{Hanle1982}. Given an independent dataset $\mathcal{D} = \{(\bx_t, y_t)\}_{t=1}^N$, the data can be partitioned into positive and negative subsets:
$$S^+=\{\bx_i, \text{if } y_i = +1\}, \quad  S^-=\{\bx_i, \text{if } y_i = -1\}.$$
where $N^+ = |S^+|$ and $N^- = |S^-|$ denote the number of positive and negative instances, respectively. The empirical AUC is then computed as:
\begin{align*}
AUC(f) = \frac{1} {N^+ N^-}\sum_{\bx_i \in S^+} \sum_{\bx_j \in S^-} \mathbb I[f(\bx_i) - f(\bx_j) > 0],
\end{align*}
where $\mathbb{I}(\cdot)$ is the indicator function.

The AUC optimization problem seeks a scoring function $f \in \mathcal{F}$ that minimizes the AUC loss, defined as $1 - \text{AUC}(f)$. Direct minimization of this loss is computationally intractable due to its NP-hard combinatorial nature. A common approach is to replace the nonconvex 0–1 loss $\mathbb I[f(\bx_i) - f(\bx_j) > 0]$ with a convex surrogate function $\gamma : \mathbb{R} \to \mathbb{R}^+$, such as the hinge or squared loss, leading to the surrogate objective:

\ben
\min_{f \in \mathcal F} L_\gamma(f): = \frac{1} {N^+ N^-}\sum_{\bx_i \in S^+} \sum_{\bx_j \in S^-} \gamma(f(\bx_i)-f(\bx_j)).
\enn
 
Batch AUC optimization suffers from inherent scalability issues due to its $\mathcal{O}(N^+ N^-)$ computational complexity. This quadratic dependence becomes prohibitive for large datasets, as each gradient computation requires processing all pairs of instances from different classes. Furthermore, storing the entire dataset for repeated pairwise comparisons imposes significant memory overhead, especially in resource-constrained settings.

Online AUC optimization (OAO) mitigates these limitations through sequential processing of data instances. In this framework, instances arrive sequentially in a stream, and the objective is decomposed into a sequence of loss terms. At each round $t$, a new instance $\bx_t$ arrives, and the current model $f_t$ predicts a score $f_t(\bx_t)$. Upon receiving the true label $y_t$, the model is updated incrementally based on the incurred loss. A distinctive aspect of OAO is that the loss at time $t$ depends on both the current instance $(\bx_t, y_t)$ and historical instances $\{(\bx_i, y_i) : i = 1, 2, \dots, t-1\}$.

Following the standard online learning framework, the regret of an OAO algorithm is defined as the difference between the cumulative loss incurred by the algorithm and that of the optimal fixed classifier chosen in hindsight~\citep{Hazan2015}. For a sequence of classifiers ${f_1, f_2, \dots, f_T}$, the AUC regret is given by:
\ben
Regret_T^{AUC} = \sum_{t=1}^\top L_t(f_t) - \min_{f \in \mathcal F}\sum_{t=1}^{T} L_t(f).
\enn
The primary objective is to design algorithms that guarantee sublinear regret, i.e.,
\ben
\lim_{T \to \infty} \frac{\text{Regret}_T^{AUC}}{T} = 0,
\enn
which ensures asymptotic convergence to the optimal performance.

\textbf{Related Works}. Online AUC optimization has emerged as a critical research direction to address the scalability limitations of batch methods for imbalanced data. Early foundational work introduced buffer-based sampling techniques to handle the pairwise nature of the AUC loss in streaming environments \citep{Zhao2011}. This was followed by efficient one-pass algorithms that maintain only first- and second-order statistics to avoid storing data \citep{Gao2016}. A significant advancement was achieved by reformulating AUC optimization as a convex-concave saddle point problem, leading to stochastic online methods with linear time and space complexity \citep{Ying2016a}. To improve convergence, subsequent work explored adaptive gradient methods that leverage historical gradient information \citep{Ding2015}, as well as adaptive moment estimation for more stable optimization \citep{liu2019}. Further theoretical advancements include the development of stochastic proximal algorithms that achieve strong convergence rates without restrictive boundedness assumptions \citep{Lei2021}. The challenge of handling nonlinear data has motivated several kernel-based approaches. These include methods that employ budgeted buffers for support vectors \citep{Hu2018} and scalable approximations using Fourier features or Nystr\"{o}m methods \citep{Ding2017}. Non-parametric approaches have also been proposed to address limitations of surrogate losses in online AUC maximization \citep{Szorenyi2016}. More recent studies have expanded the scope to high-dimensional sparse data, where efficient algorithms with reduced per-iteration cost have been developed \citep{Zhou2020}. Distributed learning scenarios have been addressed through both centralized and decentralized online AUC maximization algorithms \citep{Liu2023}. Most recently, the problem has been extended to lifelong learning settings, where novel strategies involving model decoupling and alignment have been introduced to handle sequentially arriving imbalanced tasks \citep{Zhu2023}.

\textbf{Notations for OAO}. Let $(\bx_t, y_t) \in \mathbb{R}^p \times \{-1, +1\}$ denote the $t$-th instance in the data stream, where $\bx_t$ is a $p$-dimensional feature vector and $y_t$ is the corresponding class label. For each incoming instance $(\bx_t, y_t)$, we define the opposite-class instance set $S_t = \{ (\bx_i, y_i) \mid y_i = -y_t,\ 1 \leq i < t \}$ as the collection of historical instances with opposing labels. Let $N_t = |S_t|$ denote the cardinality of $S_t$. The statistics $\bar{\bx}_t = \frac{1}{N_t} \sum_{\bx_i \in S_t} \bx_i$ and $\Sigma_t = \frac{1}{N_t} \sum_{\bx_i \in S_t} (\bx_i - \bar{\bx}_t)(\bx_i - \bar{\bx}_t)^\top$ represent the mean vector and covariance matrix of $S_t$, respectively. We assume a linear prediction function $f(\bx; \bw) = \bw^\top \bx$, where $\bw \in \mathbb{R}^p$ is the weight vector to be learned.

\section{Second-Order Surrogate Loss Function}\label{sec3}

Current online AUC optimization approaches predominantly employ two classes of surrogate loss functions: least squares and hinge losses, which exhibit fundamentally different computational characteristics. The square loss function facilitates a moment-based approach to AUC optimization, as its quadratic form admits an exact decomposition into statistical moments that fully characterize the loss landscape. This property enables square-loss optimization methods such as OPAUC to operate as purely moment-based algorithms, wherein the entire optimization process relies solely on moment statistics without requiring instance storage. In contrast, the piecewise-linear structure of the hinge loss prevents exact decomposition into statistical moments, necessitating the explicit retention of historical instances for pairwise margin computations. Hinge loss-based methods, exemplified by OAM algorithms, must maintain buffers of historical instances to compute pairwise margins. This memory dependence introduces prohibitive scalability constraints in large-scale imbalanced learning scenarios. This limitation underscores a critical trade-off: the square loss enables efficient moment-based optimization due to its regression-friendly properties, whereas the hinge loss offers superior margin maximization for classification but traditionally requires instance-wise storage.

To reconcile this fundamental trade-off, we develop a novel second-order surrogate loss framework grounded in robust optimization theory. Our key theoretical insight establishes that although the piecewise linearity of the hinge loss precludes exact moment decomposition, its worst-case aggregate behavior under moment constraints admits a closed-form upper bound expressible via first- and second-order statistics.

\subsection{Second-Order Surrogate Loss for Square-Based AUC Optimization}

OPAUC employs the pairwise square loss as its surrogate function. The gradient of the $t$-th loss term can be computed efficiently using statistical moments of the historical instance subset $S_t$\citep{Gao2016}. This property eliminates the need for storing individual instances, making the method particularly suitable for large-scale learning due to its constant memory requirement independent of dataset size. Furthermore, the $t$-th pairwise square-based AUC loss can be expressed in terms of the statistical moments of $S_t$ as follows:
\begin{equation}\label{eq:Square-AUCLoss}
\begin{split}
L_t^{\text{psl}}(\bw) &= \frac{1}{N_t} \sum_{\bx_i \in S_t} \left(1 - y_t \bw^\top (\bx_t - \bx_i)\right)^2 \\
&= \frac{1}{N_t} \sum_{\bx_i \in S_t} \left(1 - y_t \bw^\top (\bx_t - \bar{\bx}_t + \bar{\bx}_t - \bx_i)\right)^2 \\
&= \left(1 - y_t \bw^\top (\bx_t - \bar{\bx}_t)\right)^2 + \bw^\top \Sigma_t \bw.
\end{split}
\end{equation}
The decomposition in Eq.~\eqref{eq:Square-AUCLoss} demonstrates that the square-based AUC loss is completely determined by the first- and second-order statistics of the data distribution.

This key observation motivates our development of a novel optimization framework that transforms AUC maximization from reliance on pairwise comparisons to statistical moment-based computation. 
Unlike conventional approaches that replace each individual pairwise 0-1 loss term: 
\begin{equation}
\mathbb{I}(y_t(\bw^\top \bx_t > \bw^\top \bx_i)),
\end{equation} with a convex surrogate, such as $(y_t \bw^\top (\bx_t - \bx_i))^2$, our framework directly substitutes the entire aggregated pairwise loss:
$$\frac{1}{N_t}\sum_{\bx_i \in S_t} \mathbb{I}(y_t(\bw^\top \bx_t > \bw^\top \bx_i)),$$
with a novel surrogate loss $\psi_S(y_t \bw^\top (\bx_t - \bar{\bx}_t), \bw^\top \Sigma_t \bw)$. This surrogate loss function takes the form:
\begin{equation}
\psi_S(\mu, \sigma^2) = (1 - \mu)^2 + \sigma^2.
\end{equation}
where $\mu$ and $\sigma^2$ correspond to the mean and variance of the pairwise comparison outcomes. Since the loss depends explicitly on the second moment $\sigma^2 = \bw^\top \Sigma_t \bw$, we term it a second-order surrogate loss. The resulting formulation, $\psi_S$, operates exclusively on the mean vector and covariance matrix of the data, thereby entirely eliminating the need for instance-level storage.

This reformulation highlights a fundamental distinction between computational paradigms: whereas the conventional pairwise square loss requires explicit access to each historical instance $\bx_i$ for exact AUC loss calculation, our second-order surrogate loss formulation shows that only the statistical moments (mean $\bar{\bx}$ and covariance $\Sigma$) are necessary for equivalent computation. Remarkably, despite using compressed statistical representations instead of raw data, the second-order surrogate loss maintains mathematical equivalence with the original pairwise square loss—preserving all theoretical guarantees while achieving superior computational efficiency.

\subsection{Second-Order Surrogate Loss for Hinge-Based AUC Optimization}

Replacing the pairwise 0-1 loss with the pairwise hinge loss yields an instance-wise risk formulation where each incoming example $(\mathbf{x}_t, y_t)$ incurs a loss defined through margin comparisons against all historical opposing instances $\bx_i \in S_t$:
\begin{align}
L_t^{\text{hl}}(\mathbf{w}) = \frac{1}{N_t}\sum_{\mathbf{x}_i \in S_t} \max\left(0, 1 - y_t\mathbf{w}^\top(\mathbf{x}_t - \mathbf{x}_i)\right).
\end{align}
However, the piecewise-linear structure of the hinge function prevents its decomposition into summary statistics, necessitating explicit storage of the entire set $S_t$ of historical opposing instances. This requirement leads to $\mathcal{O}(N_t p)$ memory complexity, which fundamentally undermines the scalability objectives of online learning—particularly in high-dimensional feature spaces (large $p$) or long data streams (growing $N_t$). Despite its theoretical advantages for classification, the direct use of the hinge loss thus incurs prohibitive memory costs in large-scale applications.

Inspired by the efficiency of moment-based AUC optimization, we propose a novel approach to characterize the hinge-based AUC loss through its statistical moments. Consider an alternative historical opposing instances dataset $S_t'$ that preserves the cardinality $N_t$, mean vector $\bar{\mathbf{x}}_t$, and covariance matrix $\Sigma_t$ of the original opposing set $S_t$. Within this family of moment-consistent datasets, we formulate the worst-case hinge loss as the solution to the following constrained optimization problem:

\begin{equation}\label{eq:Robust-AUCLoss}
\begin{split}
\max_{S_t'} \quad & \frac{1}{|S_t^{'}|}\sum_{\mathbf{x}_i^{'} \in S_t^{'}} \max\left(0, 1 - y_t\mathbf{w}^\top(\mathbf{x}_t - \mathbf{x}_i^{'})\right) \\
\text{s.t.} \quad & |S_t^{'}| = N_t \\
& \frac{1}{|S_t'|}\sum_{\mathbf{x}_i^{'} \in S_t^{'}} \mathbf{x}_i^{'} = \bar{\mathbf{x}}_t \\
& \frac{1}{|S_t'|}\sum_{\mathbf{x}_i^{'} \in S_t^{'}} (\mathbf{x}_i^{'} - \bar{\mathbf{x}}_t)(\mathbf{x}_i^{'} - \bar{\mathbf{x}}_t)^\top = \Sigma_t.
\end{split}
\end{equation}

This formulation seeks the maximum hinge loss over all hypothetical datasets sharing the same first two moments as $S_t$, effectively constructing a finite-sample analogue of distributionally robust optimization under moment constraints \citep{Delage2010}. By focusing on the worst-case scenario within this feasible set, we obtain a robust upper bound that depends solely on $(\bar{\mathbf{x}}_t, \Sigma_t)$ rather than individual instances, enabling memory-efficient optimization while preserving the geometric properties of the hinge loss.

\begin{theorem}[Moment-Constrained Upper Bound for Hinge-based AUC Loss] 
\label{The-H-B-AUC}
\quad\\Let $\mathcal{D} = \{(\mathbf{x}_i, y_i)\}_{i=1}^n \subset \mathbb{R}^p \times \{-1, +1\}$ be a set of training instances with labels opposite to a given example $(\mathbf{x}, y)$, and let $\mathbf{w} \in \mathbb{R}^p \setminus {\mathbf{0}}$ be an arbitrary weight vector. Denote the empirical mean and covariance matrix of the feature vectors $\{\mathbf{x}_i\}_{i=1}^n$ by
$$ \bar \bx = \frac{1}{n} \sum_{i=1}^n \bx_i, \Sigma = \frac{1}{n} \sum_{i=1}^n (\bx_i - \bar\bx)(\bx_i - \bar\bx)^\top.$$
Then, the hinge-based AUC loss admits the following upper bound:
\begin{align}\label{eq:Hinge-AUCLoss}
\begin{split}
\frac{1}{n}\sum_{i=1}^n \mathbb{I}[y\mathbf{w}^\top(\mathbf{x}-\mathbf{x}_i) < 0] &\leq \frac{1}{n}\sum_{i=1}^n \max(0,1-y\mathbf{w}^\top(\mathbf{x} - \mathbf{x}_i)) \\ 
&\leq \frac{1}{2}\left(1-y\mathbf{w}^\top(\mathbf{x}-\bar{\mathbf{x}}) + \sqrt{(1-y\mathbf{w}^\top(\mathbf{x}-\bar{\mathbf{x}}))^2 + \mathbf{w}^\top\Sigma\mathbf{w}} \right).
\end{split}
\end{align}
\end{theorem}

The proof relies on the following key lemma, which establishes an upper bound for the average hinge loss under moment constraints.

\begin{lemma}[Constrained Hinge Loss Bound] 
\label{lemma-Bound-H}
Let $n$ be a positive integer and let $\mu \in \mathbb{R}$, $\sigma > 0$ be given real numbers. Define the set
\begin{equation}
\mathcal{C} = \left\{\mathbf{c}=(c_1,c_2,\cdots,c_n) \in \mathbb{R}^n \mid \frac{1}{n}\sum_{i=1}^n c_i = \mu,\ \frac{1}{n}\sum_{i=1}^n (c_i - \mu)^2 = \sigma^2 \right\}.
\end{equation}
For any $\mathbf{c} \in \mathcal{C}$, define the average hinge loss as
$$\ell(\mathbf{c}) = \frac{1}{n}\sum_{i=1}^n \max(0,1-c_i).$$
Then, for all $\mathbf{c} \in \mathcal{C}$, we have
\begin{equation}
\ell(\mathbf{c}) \leq \psi_M(\mu,\sigma^2) = \frac{1}{2}\left[(1-\mu) + \sqrt{(1-\mu)^2 + \sigma^2}\right].
\end{equation}
Moreover, this bound can be equivalently expressed as
\begin{equation}
\psi_M(\mu,\sigma^2) = \Phi_M(v)(1-\mu) + \phi_M(v)\sigma,
\end{equation}
where $v = \frac{1-\mu}{\sigma}$ and 
\begin{align}
\Phi_M(v) = \frac{1}{2}\left(1 + \frac{v}{\sqrt{1+v^2}}\right), \
\phi_M(v) = \frac{1}{2}\sqrt{\frac{1}{1+v^2}}.
\end{align}
\end{lemma}

\begin{proof}[Proof of Theorem 1]
For each instance $\mathbf{x}_i$, define
$$ c_i = y\bw^\top(\bx - \bx_i) .$$
The mean and variance of the $c_i$ values are given by
$$ \mu_c = \frac{1}{n} \sum_{i=1}^n c_i = \frac{1}{n}  \sum_{i=1}^n  y\bw^\top(\bx - \bx_i) =  y\bw^\top(\bx - \bar \bx),$$
and
$$ \sigma_c^2 = \frac{1}{n}  \sum_{i=1}^n (c_i - \mu_c)^2 = \bw^\top \Sigma \bw.$$
Applying lemma \ref{lemma-Bound-H} with  $\mu = \mu_c$, and $\sigma^2 = \sigma_c^2$, we obtain
\begin{equation}
\frac{1}{n}\sum_{i=1}^n \max(0,1-c_i) \le  \frac{1}{2}\left[(1-\mu_c) + \sqrt{(1-\mu_c)^2 + \sigma_c^2}\right],
\end{equation}
which is exactly the desired inequality.
\end{proof}

Building upon the moment-based upper bound established in Theorem 1, we propose a Second-Order Surrogate Loss function defined as:
\ben
\psi_M(y \mathbf{w}^\top( \bx-\bar \bx), \mathbf{w}^\top\Sigma \mathbf{w}) =  \half \left(1-y \mathbf{w}^\top(\bx-\bar \bx) + \sqrt{(1-y \mathbf{w}^\top(\bx -\bar \bx))^2 + \mathbf{w}^\top\Sigma \mathbf{w}} \right).
\enn
$\psi_M$ represents a fundamental advancement beyond traditional pairwise loss formulations by enabling moment-based optimization while preserving the margin-maximization properties essential for classification tasks.

\subsection{Properties of the Second-Order Surrogate Loss}

We now analyze key properties of the proposed second-order surrogate loss functions. These properties are essential for establishing convergence guarantees and understanding the behavior of the resulting optimization algorithms.

The function $\psi_M(\mu, \sigma^2)$ can be viewed as a smooth approximation to the hinge loss, with its smoothness controlled by the variance parameter $\sigma^2$ \citep{Luo2021}.

\begin{lemma}[Approximation Properties of $\psi_M$]
\label{lemm-app-M}
The function $\psi_M(\mu, \sigma^2)$ satisfies:
$$0 \le \psi_{M}(\mu,\sigma^2) - \max\left(0,1-\mu\right) \le \frac{\sigma}{2}.$$
Furthermore, $\psi_M(\mu, \sigma^2)$ converges uniformly to the hinge loss in $\mu$ as $\sigma$ tends to 0.
\end{lemma}
This lemma establishes that $\psi_M$ provides a close approximation to the original hinge loss, with the approximation error bounded by $\sigma/2$. The uniform convergence property ensures that in the limit of zero variance, we recover the standard hinge loss.



For the hinge-based surrogate function, we establish the following key properties:

\begin{lemma}[Properties of $\Psi_M$]\label{lemm-pro-M}
Let $\bar{\bx}$ and $\Sigma$ be the mean and covariance matrix of a subset $S = \{\bx_i\}_{i=1}^n$, $\bx_i \in \mathbb{R}^p$, respectively. Given an example $(\bx, y)$, define $\Psi_M: \mathbb{R}^p \to \mathbb{R}$ as
\ben
\Psi_M(\bw) = \psi_M(y \bw^\top(\bx -\bar \bx), \bw^\top \Sigma \bw).
\enn 
Then $\Psi_M(\bw)$ satisfies:
\begin{enumerate}
\item $\Psi_M(\bw)$ is differentiable with gradient:
$$\nabla \Psi_{M}(\bw) = - \Phi_M(v) y(\bx-\bar \bx)  + \phi_M(v) \frac{\Sigma \bw}{\sqrt{\bw^\top\Sigma \bw}},$$
where $v = (1 - y\bw^\top(\bx - \bar{\bx}))/\sqrt{\bw^\top\Sigma\bw}$. 
\item $\Psi_M(\bw)$ is a convex function.
\item If $\|\bx'\| \leq 1$ for all $\bx' \in \mathbb R^p$, then $\|\nabla \Psi_M(\bw)\| \leq 3$.
\end{enumerate}
\end{lemma}

The bounded gradient property established in part (3) is particularly important for online learning applications, as it ensures stability during the optimization process and facilitates the derivation of regret bounds.

These properties collectively demonstrate that the proposed second-order surrogate losses maintain the desirable characteristics of their pairwise counterparts while enabling efficient moment-based optimization. The convexity guarantees convergence to global optima, while the smoothness and bounded gradient properties ensure stable and efficient optimization in online learning scenarios.

\section{Online AUC Maximization Based on Second-Order Surrogate Loss}\label{sec4}

In this section, we address the problem of learning a linear classifier through online minimization of the second-order surrogate loss. For each arriving instance $(\mathbf{x}_t, y_t)$, the incurred AUC loss depends on a dynamically updated comparison set $S_t = \{\mathbf{x}_i : i \in I_t\}$ consisting of historical instances with opposing labels. By leveraging second-order statistics, we compress the entire set $S_t$ into its mean vector and covariance matrix:
\ben
\bar \bx_t = \frac{1}{|S_t|} \sum_{i \in I_t} \bx_i, \  \Sigma_t = \frac{1}{|S_t|} \sum_{i \in I_t} (\bx_i - \bar \bx_t)(\bx_i - \bar \bx_t)^\top, 
\enn
which allows the loss at time $t$ to be compactly expressed as $\Psi_t(\mathbf{w}) = \psi\left(y_t \mathbf{w}^\top(\mathbf{x}_t - \bar{\mathbf{x}}_t), \mathbf{w}^\top \Sigma_t \mathbf{w}\right)$. Here, $\psi$ denotes either the proposed second-order hinge-based surrogate $\psi_M$ or the square-based surrogate $\psi_S$.

To mitigate overfitting and enhance generalization, we incorporate Tikhonov regularization $R(\mathbf{w}) = \frac{1}{2} \lambda \|\mathbf{w}\|^2$, resulting in the composite objective function at each time step:
$$L_t(\bw) = \frac12 \lambda \|\bw\|^2 + \Psi_t(\mathbf{w}).$$
The performance of an online learning algorithm is evaluated through the cumulative AUC regret, defined as the difference between the total loss incurred by the algorithm and that achieved by the optimal fixed classifier selected in hindsight:
\ben
Regret_T^{AUC} = \sum_{t=1}^T L_t(\bw_t) -\min_{\bw \in \mathbb R^p} \sum_{t=1}^T L_t(\bw),
\enn
Our goal is to design an efficient online learning algorithm that attains sublinear regret, i.e., $\text{Regret}_T^{\text{AUC}} = o(T)$, thereby ensuring asymptotic convergence to the optimal performance.

\subsection{Online Gradient Descent Algorithm}

Various algorithms exist to achieve low regret in online learning. The Online Gradient Descent (OGD) algorithm, introduced by \citet{Zinkevich2003}, extends standard gradient descent to the online setting. At each iteration, OGD updates the model by moving in the direction of the gradient of the immediate loss, followed by a projection onto a feasible convex set. For general convex loss functions, OGD achieves an $\mathcal{O}(\sqrt{T})$ regret bound. When the loss function is $\lambda$-strongly convex, a step size schedule of $\eta_t = \frac{1}{\lambda t}$ yields a tighter $\mathcal{O}(\ln T)$ convergence rate \citep{Hazan2007}. Other efficient online methods include the Online Newton Method \citep{Hazan2007}, Follow the Regularized Leader (FTRL) \citep{Kalai2003}, and Online Mirror Descent \citep{Orabona2015}.

In this work, we employ the OGD framework to learn a sequence of classifiers $\{\mathbf{w}_1, \mathbf{w}_2, \cdots, \mathbf{w}_{T}\}$. Algorithm \ref{alg:oaucsosl} outlines the proposed online AUC optimization method using second-order surrogate loss (OAUC-SOSL). The update rule is given by:
\begin{equation}\label{opaucupdaterule}
\mathbf{w}_{t+1} = \mathbf{w}_t - \eta_t \nabla L_t(\mathbf{w}_t) = (1 - \eta_t \lambda) \mathbf{w}_t - \eta_t \nabla \psi \left( y_t \mathbf{w}_t^\top (\mathbf{x}_t - \bar{\mathbf{x}}_t); \mathbf{w}_t^\top \Sigma_t \mathbf{w}_t \right),
\end{equation}
where $\eta_t$ is a predefined step size. 

\begin{algorithm}[h!]
  \renewcommand{\algorithmicrequire}{\textbf{Input:}}
  \renewcommand{\algorithmicensure}{\textbf{Output:}}
  \caption{A Framework for Online AUC Optimization Based on Second-Order Surrogate Loss(OAUC-SOSL)}
  \label{alg:oaucsosl}
  \begin{algorithmic}[1]
    \REQUIRE The regularization parameter $\lambda \ge 0 $, the step size $\{\eta_t\}_{t=1}^T$.
    \STATE Initialize $\bw_1 = \mathbf{0}$,  $N_0^+=N_0^-=0, \bar \bx_0^+ = \bar \bx_0^-=\mathbf{0}_{p \times 1}, \Sigma_0^- = \Sigma_0^+ = \mathbf{0}_{p \times p}$.
    \FOR { $t= 1,2,...,T$}
    \STATE Receive a training example $(\bx_t,y_t)$;
    \IF {$y_t = 1$}
    \STATE $N_t^+ = N_{t-1}^+ + 1$ and $N_t^- = N_{t-1}^-$;
    \STATE Update $\bar \bx_t^+$ by \eqref{UpmuP} and $\bar \bx_t^- = \bar \bx_{t-1}^-$;
    \STATE Update $\Sigma_t^+$ by \eqref{UpsiP} and $\Sigma_t^- = \Sigma_{t-1}^- $;
    \STATE Assign $\bar \bx_t = \bar \bx_t^-$, and $\Sigma_t = \Sigma_t^-$;
    \ELSE
    \STATE $N_t^- = N_{t-1}^- + 1$ and $N_t^+ = N_{t-1}^+$;
    \STATE Update $\bar \bx_t^-$ by \eqref{UpmuN} and $\bar \bx_t^+ = \bar \bx_{t-1}^+$;
    \STATE Update $\Sigma_t^-$ by \eqref{UpsiN} and $\Sigma_t^+ = \Sigma_{t-1}^+ $;
    \STATE Assign $\bar \bx_t = \bar \bx_t^+$, and $\Sigma_t = \Sigma_t^+$;
    \ENDIF
    \STATE Update the classifier with online gradient descent method \eqref{opaucupdaterule}.
    \ENDFOR
  \end{algorithmic}
\end{algorithm}

We consider two variants of the algorithm.

OAUC-S: Uses the square-based second-order surrogate loss function $\psi = \psi_S$, with gradient:

\begin{equation}\label{opaucupdateruleSSQ}
\nabla \psi_S\left( y_t \mathbf{w}_t^\top (\mathbf{x}_t - \bar{\mathbf{x}}_t); \mathbf{w}_t^\top \Sigma_t \mathbf{w}_t \right) = (1-y_t\bw_t^\top(\bx_t-\bar \bx_t))(-y_t(\bx_t-\bar \bx_t)) + \Sigma \bw_t.
\end{equation}
OAUC-M: Uses the hinge-based square-based second-order surrogate loss function $\psi = \psi_M$, with gradient:
\begin{equation}\label{opaucupdateruleMM}
\nabla \psi_M\left( y_t \mathbf{w}_t^\top (\mathbf{x}_t - \bar{\mathbf{x}}_t); \mathbf{w}_t^\top \Sigma_t \mathbf{w}_t \right) = -\Phi_M(v_t) y_t (\mathbf{x}_t - \bar{\mathbf{x}}_t) + \frac{\phi_M(v_t)}{\sqrt{\mathbf{w}_t^\top \Sigma_t \mathbf{w}_t}} \Sigma_t \mathbf{w}_t,
\end{equation}
where $v_t = \left(1 - y_t \mathbf{w}_t^\top (\mathbf{x}_t - \bar{\mathbf{x}}_t)\right) / \sqrt{\mathbf{w}_t^\top \Sigma_t \mathbf{w}_t}$.

This framework efficiently leverages second-order statistics to enable scalable online AUC optimization with strong theoretical guarantees.

\subsection{Online Update of First and Second-Order Statistics}
The loss incurred at each time step $t$ depends on the statistical moments derived from previously observed data. To facilitate efficient computation, we maintain two distinct sets of statistics corresponding to positive and negative instances. Let $S_t^+ = \{\mathbf{x}_i : y_i = +1, i \leq t\}$ and $S_t^- = \{\mathbf{x}_i : y_i = -1, i \leq t\}$ represent the sets of positive and negative instances observed up to time $t$, with cardinalities $N_t^+ = |S_t^+|$ and $N_t^- = |S_t^-|$, respectively. Instead of storing these sets explicitly, we maintain and recursively update their first- and second-order statistical moments. The initial values for these moments are set to:
$$\bar \bx_0^- = \bar \bx_0^+ = {\mathbf{0}_{p \times 1}}, \Sigma_0^- = \Sigma_0^+ = \mathbf{0}_{p \times p}.$$
These statistical moments are updated incrementally upon the arrival of each new instance. For a new instance $(\mathbf{x}_t, y_t)$, the relevant statistics are updated according to its class label.

When $y_t = +1 $, the AUC loss computation relies on the statistics of negative instances. Accordingly, we set $S_t = S_t^-$, $\bar{\mathbf{x}}_t = \bar{\mathbf{x}}_t^-$, and $\Sigma_t = \Sigma_t^-$. The statistics for the positive class are updated using the following recursive relations:
\begin{align}\label{UpmuP}
\bar{\mathbf{x}}_t^+ = \bar{\mathbf{x}}_{t-1}^+ + \frac{1}{N_t^+}(\mathbf{x}_t - \bar{\mathbf{x}}_{t-1}^+), N_t^+ = N_{t-1}^+ + 1,
\end{align}
and 
\begin{align}\label{UpsiP}
\Sigma_t^+ = \Sigma_{t-1}^+ + \bar{\mathbf{x}}_{t-1}^+[\bar{\mathbf{x}}_{t-1}^+]^\top - \bar{\mathbf{x}}_{t}^+[\bar{\mathbf{x}}_{t}^+]^\top + (\mathbf{x}_t\mathbf{x}_t^\top - \Sigma_{t-1}^+ - \bar{\mathbf{x}}_{t-1}^+[\bar{\mathbf{x}}_{t-1}^+]^\top)/N_t^+.
\end{align}

When $y_t = -1$, the AUC loss computation depends on the statistics of positive instances. Hence, we set $S_t = S_t^+$, $\bar{\mathbf{x}}_t = \bar{\mathbf{x}}_t^+$, and $\Sigma_t = \Sigma_t^+$. The statistics for the negative class are updated as follows:

\begin{align}\label{UpmuN}
\bar{\mathbf{x}}_t^- = \bar{\mathbf{x}}_{t-1}^- + \frac{1}{N_t^-}(\mathbf{x}_t - \bar{\mathbf{x}}_{t-1}^-), N_t^- = N_{t-1}^- + 1
\end{align}
and
\begin{align}\label{UpsiN}
\Sigma_t^- = \Sigma_{t-1}^- + \bar{\mathbf{x}}_{t-1}^-[\bar{\mathbf{x}}_{t-1}^-]^\top - \bar{\mathbf{x}}_{t}^-[\bar{\mathbf{x}}_{t}^-]^\top + (\mathbf{x}_t\mathbf{x}_t^\top - \Sigma_{t-1}^- - \bar{\mathbf{x}}_{t-1}^-[\bar{\mathbf{x}}_{t-1}^-]^\top)/N_t^-.
\end{align}

This systematic updating procedure ensures that all requisite statistical moments are accurately maintained while requiring only $\mathcal{O}(p^2)$ memory, thereby rendering the approach particularly suitable for large-scale online learning applications.

\subsection{Regret Analysis of OAUC-SOSL Algorithms}\label{sec5}

In this section, we present the regret bounds for the proposed OAUC-SOSL algorithms. Our analysis builds upon established techniques in online optimization theory~\citep{Gao2016,Zinkevich2003,Hazan2007}. We first recall the regret bound for the square-based AUC loss variant OAUC-S (equivalent to OPAUC), originally established by \citet{Gao2016}.

\begin{theorem}[Regret Bound for OAUC-S]\label{RegretofOPAUC}
Suppose $\|\bx_t\| \le 1$ for all t, and define the optimal classifier $w^\star$and minimal loss $L^\star$ as
\be\label{wstar}
\bw^{\star} = \arg \min_\bw \sum_{t=1}^\top \frac{1}{2} \lambda \|\bw\|^2 + \psi_S (y_t \bw^\top(\bx_t-\bar \bx_t); \bw^\top \Sigma_t \bw),
\en
and 
\ben
L^{\star} =  \sum_{t=1}^\top \frac{1}{2} \lambda \|\bw^{\star}\| + \psi_S  (y_t{\bw^\star}^\top(\bx_t-\bar \bx_t); {\bw^\star}^\top \Sigma_t \bw^\star).
\enn
For any $\lambda > 0$, the regret of the OAUC-S algorithm  satisfies
\ben
\text{Regret}_{T}^{AUC} \le \frac{2(4+\lambda)}{\lambda} + 2\sqrt{\frac{4+\lambda}{\lambda} TL^{\star}},
\enn
when using the step size $\eta_t = \frac{1}{4+\lambda+\sqrt{(4+\lambda)^2 + (4+\lambda)\lambda TL^{\star}}}.$
\end{theorem}

The proof follows ~\cite{Gao2016} and is omitted here. While this $O(\sqrt{T})$ bound is achieved through the smoothness of $\psi_S$, the optimal step size depends on the unknown minimal loss $L^{\star}$, requiring cross-validation in practice.

We now present our main theoretical result for the hinge-based variant OAUC-M.

\begin{theorem}[Regret Bound for OAUC-M]\label{regretofOAUC2}
Suppose $\|\bx_t\| \le 1$ for all t, and define
\begin{align*}
 \bw^{\star} = \arg \min_{\bw \in \mathbb R^p} \sum_{i=1}^\top \frac{1}{2} \lambda \|\bw\|^2 + \psi_M(y_t \bw^\top(\bx_t - \bar \bx_t); \bw^\top \Sigma_t \bw)
\end{align*}
For any $\lambda > 0$, the regret of OAUC-M satisfies
\ben
\text{Regret}_{T}^{AUC} \le \frac{18 (1 + \ln T)}{\lambda}.
\enn
\end{theorem}

Our analysis demonstrates that OAUC-M achieves convergence to the optimal batch classifier with an $O(\ln T)$ regret bound—the first such result in online AUC optimization literature. This accelerated convergence stems from the bounded gradient property of the second-order surrogate loss $\psi_M$, which ensures stable optimization dynamics.

\section{Online nonlinear AUC maximization}\label{sec6}

\subsection{online kernel AUC maximization method}

While the previous section focused on linear classifiers, many real-world applications require nonlinear decision boundaries. We extend our framework to kernel-based learning, building upon but significantly improving the Kernelized Online Imbalanced Learning (KOIL) approach proposed by \citet{Hu2018}. 

Let $\mathcal{H}$ be a Reproducing Kernel Hilbert Space (RKHS) with kernel function $k: \mathcal{X} \times \mathcal{X} \rightarrow \mathbb{R}$. We maintain two buffers $\mathcal{K}_t^+$ and $\mathcal{K}_t^-$ to store positive and negative support vectors, with corresponding index sets $I_t^+$ and $I_t^-$. The nonlinear classifier at time $t$ takes the form:
\be \label{kernelft}
f_t(\bx) = \sum_{i \in I_t^- \cup I_t^+}\alpha_{i,t}k(\bx_i,\bx),
\en
where $\alpha_{i,t}$ are the learned weights.

For each new instance $(\bx_t, y_t)$, we compute the AUC loss by comparing it with the appropriate buffer ($S_t = \mathcal{K}_t^-, I_t = I_t^-$ if $y_t = +1$, $S_t = \mathcal{K}_t^+, I_t = I_t^+$ if $y_t = -1$). Using our second-order surrogate loss, the regularized objective becomes:
\be
L_t(f_t) = \frac 12 \lambda  \|f_t\|_{\mathcal H}^2 +  \psi(y_t(f_t(\bx_t)-\mu_{t});\sigma_{t}^2),
\en
where $\mu_{t} = \frac{1}{|S_t|}\sum_{i \in I_t} f_t(\bx_i)$ and $\sigma_{t}^2 = \frac{1}{|S_t|}\sum_{i \in I_t} (f_t(\bx_i) - \mu_{t})^2$ represent the mean and variance of predictions on the comparison set.

Algorithm \ref{alg-okauc} outlines the OKAUC-SOSL framework.
\begin{algorithm}[h!]
  \renewcommand{\algorithmicrequire}{\textbf{Input:}}
  \renewcommand{\algorithmicensure}{\textbf{Output:}}
  \caption{A Framework for Online Kernel AUC Optimization Based on the Second-Order Surrogate Loss(OKAUC-SOSL)}
  \label{alg-okauc}
  \begin{algorithmic}[1]
    \REQUIRE The regularization parameter $\lambda > 0 $, the step size $\{\eta_t\}_{t=1}^T$.
    \ENSURE Update a nonlinear classifier $f_{t+1}$.
    \STATE Initialize positive and negative buffer $\mathcal K_1^- = \mathcal K_1^+ = \emptyset$, which have a  fixed $N^+$ and $N^-$ budget, respectively.  And the corresponding weight of the initial kernel-based classifier $\alpha_1 = (\alpha_{i,1})$, where $i \in I_1^- \cup I_1^+$ and $\alpha_{i,1}=0$. \\
    \FOR { $t= 1,2,...,T$}
    \STATE Receive a training example ${\mathbf z}_t = (\bx_t,y_t)$;
    \IF {$y_t = +1$}
    \STATE Let $I_t = I_t^-$;
    \STATE $\alpha_{t+1}^{'}$ = UpdateClassifier(${\mathbf z}_t,\mathcal K_t^-,\mathcal K_t^+,\lambda,\eta_t,\alpha_t, I_t$);
    \STATE [$\mathcal K_{t+1}^+, I_{t+1}^+, \alpha_{t+1} $] = UpdateBuffer(${\mathbf z}_t, \mathcal K_t^+, N^+, I_t^+, \alpha_{t+1}^{'}$);
    \STATE $\mathcal K_{t+1}^- = \mathcal K_{t}^-, I_{t+1}^- = I_{t}^- $;    
    \ELSE
    \STATE Let $I_t = I_t^+$;
    \STATE $\alpha_{t+1}^{'}$ = UpdateClassifier(${\mathbf z}_t,\mathcal K_t^-,\mathcal K_t^+,\lambda,\eta_t,\alpha_t, I_t$);
    \STATE [$\mathcal K_{t+1}^-, I_{t+1}^-, \alpha_{t+1} $] = UpdateBuffer(${\mathbf z}_t, \mathcal K_t^-, N^-, I_t^-, \alpha_{t+1}^{'}$);
    \STATE $\mathcal K_{t+1}^+ = \mathcal K_{t}^+, I_{t+1}^+ = I_{t}^+ $. 
    \ENDIF
    \ENDFOR
  \end{algorithmic}
\end{algorithm}

This framework operates through two core procedures at each iteration: the UpdateClassifier step (Algorithm \ref{alg-3}) and the UpdateBuffer step (Algorithm \ref{alg-4}). During the UpdateClassifier step, the weight vector $\alpha$ is updated using gradient information derived from the second-order surrogate loss. The UpdateBuffer step maintains the support vector buffers within fixed budgets $N^+$ and $N^-$ to ensure computational efficiency. This approach preserves the theoretical advantages of second-order surrogate loss functions while enabling effective nonlinear classification, overcoming the computational limitations of traditional kernel methods through efficient buffer management and moment-based optimization.

\subsection{Update Classifier}
We apply Online Gradient Descent(OGD) method to update the decision function at each trial~\citep{WangZ2010,Crammer2004,Dekel2008,Kivinen2004}. That is
\be\label{updaterule1}
f_{t+1} = f_{t} - \eta_t \nabla L_t(f_t)=(1-\lambda \eta_t)f_t - \eta_t \nabla \psi(y_t(f_t(\bx_t)-\mu_{t});\sigma_{t}^2).
\en
\noindent

\begin{enumerate}
\item Let the second-order surrogate loss be $\psi = \psi_{M}$, our algorithm is also named online kernel AUC maximization algorithm based on the second-order surrogate loss $\psi_{M}$(OKAUC-M). 
The t-th regularized loss is given as
\ben
L_t(f_t) = \half \lambda \|f_t\|_{\mathcal H}^2 + \frac 12 \left(1-y_t(f_t(\bx_t) -\mu_{t}) + \sqrt{(1-y_t(f_t(\bx_t) -\mu_{t}))^2 + \sigma_{t}^2} \right).
\enn
By using the rule of operation in RKHS, we have
\ben
f_t(\bx) = \left< f_t,k(.,\bx)\right> ,\  \nabla f_t(\bx) = k(.,\bx).
\enn
Then we can calculate the gradient of $\mu_{t}$ and $\sigma_{t}^2$. That is 
\begin{align*}
\nabla \mu_{t}  = \frac 1{N_t} \sum_{i \in I_t} k(.,\bx_i),
\qquad \nabla \sigma_{t}^2  = \frac 2{N_t}  \sum_{i \in I_t} (f_t(\bx_i)-\mu_{t})k(.,\bx_i).
\end{align*}
Denote by
\be\label{A}
A_t = (1-y_t(f_t(\bx_t)-\mu_{t}))^2 + \sigma_{t}^2, \ b_t = 1-y_t(f_t(\bx_t)-\mu_{t}), 
\en
and
\be\label{PsiMM}
\Psi_t = \frac12(b_t + \sqrt{A_t}).
\en
Then we can calculate $\nabla L_t(f_t)$ as
\begin{align}
\begin{split}
&\nabla L_t(f_t) \\
&= \lambda f_{t}+ \frac 12 \left( \left(1+ \frac{b_t}{\sqrt{A_t}}\right)\left(-y_t \left(k(.,\bx_t) - \frac 1{N_t} \sum_{i \in I_t} k(.,\bx_i)\right)\right) + \frac{1}{N_t\sqrt{A_t}} \sum_{i \in I_t} (f_t(\bx_i)-\mu_{t})k(.,\bx_i) \right) \\
&= \lambda f_{t}+ \frac 12 \left(-y_t \left(1+ \frac{b_t}{\sqrt{A_t}}\right)k(.,\bx_t) \right) + \frac{1}{2 N_t\sqrt{A_t}} \left(\sum_{i \in I_t} \left(y_t \left(\sqrt{A_t}+b_t\right) + f_t(\bx_i)-\mu_{t}\right)k(.,\bx_i)    \right)   \\
&= \lambda f_{t} - \frac{1}{\sqrt{A_t}} y_t \Psi_t k(.,\bx_t) +  \left(\sum_{i \in I_t} \frac{1}{N_t\sqrt{A_t}}\left( y_t \Psi_t  + \frac12(f_t(\bx_i)-\mu_{t})\right)k(.,\bx_i)  \right) 
\end{split}
\end{align}
We update the classifier with an OGD step. That is 
\ben
f_{t+1}^{'} &&= f_t -\eta_t \nabla L_t(f_{t}) \\
&&= (1-\lambda\eta_t) f_t + \frac{1}{\sqrt{A_t}} \eta_t y_t \Psi_t k(.,\bx_t) -  \left(\sum_{i \in I_t} \frac{1}{N_t\sqrt{A_t}} \eta_t \left( y_t \Psi_t  + \frac12 (f_t(\bx_i)-\mu_{t})\right)k(.,\bx_i)  \right) .
\enn
The correspondence weights of support vectors are given as
\be\label{updatealpha}
\alpha_{i,{t+1}}^{'} = 
\begin{cases}
\frac{1}{\sqrt{A_t}} \eta_t y_t \Psi_t  &\  \text{if} \  i = t+1 \\
(1-\lambda \eta_t)\alpha_{i,{t}} - \frac{1}{ N_t\sqrt{A_t}} \eta_t \left( y_t \Psi_t  + \frac12 (f_t(\bx_i)-\mu_{t})\right) &\  \text{if} \  i \in I_t \\
(1-\lambda \eta_t)\alpha_{i,{t}}   &\  \text{otherwise}.
\end{cases}
\en

\item Let the second-order surrogate loss be $\psi = \psi_{S}$, our algorithm is also named online kernel AUC maximization algorithm based on the second-order surrogate loss $\psi_{S}$(OKAUC-S). The $t$-th regularized loss is given as
\ben
L_t(f_t) = \frac12 \lambda \|f_t\|_{\mathcal H}^2 + \frac 12\left((1-y_t(f_t(\bx_t) -\mu_{t}))^2 + \sigma_{t}^2 \right).
\enn
According to the rule of operation in RKHS, we can calculate $\nabla L_t(f_t)$ as
\begin{align*}
\nabla L_t(f_t) &= \lambda f_{t} +  \left((1-y_t(f_t(\bx_t) -\mu_{t}))(-y_t(k(.,\bx_t)-\frac 1{N_t} \sum_{i \in I_t} k(.,\bx_i))) + \frac 1{N_t} \sum_{i \in I_t}(f_t(\bx_i)-\mu_{t})k(.,\bx_i) \right) \\
& = \lambda f_{t} - y_t b_t k(.,\bx_t) +\sum_{i \in I_t} \frac 1{N_t} \left(y_t b_t + f_t(\bx_i)-\mu_{t} \right) k(.,\bx_i).
\end{align*}
\end{enumerate}
Similarly, we update $f_t$ with a OGD rule, $f_{t+1}^{'} = f_t - \eta_t \nabla L_t(f_{t})$. The correspondence weights of the support vectors are
{\large \begin{align}\label{updatealphassq}
\alpha_{i,{t+1}}^{'} =
\begin{cases}
  \eta_t  y_t b_t  &\qquad \text{if} \  i = t+1 \\
  (1-\lambda \eta_t)\alpha_{i,{t}} - \frac {1}{N_t}\eta_t \left(y_t b_t + f_t(\bx_i)-\mu_{t} \right)  &\qquad \text{if} \  i \in I_t \\
  (1-\lambda \eta_t)\alpha_{i,{t}}   &\qquad \text{otherwise}.
\end{cases}
\end{align}
}
\begin{algorithm}
  \renewcommand{\algorithmicrequire}{\textbf{Input:}}
  \renewcommand{\algorithmicensure}{\textbf{Output:}}
  \caption{UpdateClassifier(${\mathbf z}_t,\mathcal K_t^-,\mathcal K_t^+,\lambda,\eta_t,\alpha_t, I_t$);}
  \label{alg-3}
  \begin{algorithmic}[1]
    \REQUIRE Given ${\mathbf z}_t,\mathcal K_t^-,\mathcal K_t^+,\lambda,\eta_t,\alpha_t$ and $I_t$;
    \ENSURE Update the corresponding weight of the support vectors;
    \STATE  Compute $f_t(\bx_t)$ and $f_t(\bx_i),i \in I_t$ by Eq \eqref{kernelft};
    \STATE  Compute $A_t, b_t$ by Eq \eqref{A}, computer $\Psi_t$ by Eq \eqref{PsiMM}; \\
    \STATE  Update the classifier with a rule \eqref{updaterule1}, and the corresponding weights is updated with one of the following rule: \\    
          update $\alpha_{i,t+1}^{'}$ by Eq \eqref{updatealpha} [OKAUC-M];\\
          update $\alpha_{i,t+1}^{'}$ by Eq \eqref{updatealphassq} [OKAUC-S];
    \RETURN $\alpha_{t+1}^{'}=(\alpha_{i,t+1}^{'}), i \in I_t^- \cup I_t^+$. 
  \end{algorithmic}
\end{algorithm}

\subsection{Update Buffer}

For computation efficiency, we fix the buffer size of $\mathcal K_t^+$ and $\mathcal K_t^-$ as in many online kernel-based algorithms~\citep{WangZ2010,Hu2018,Crammer2004,Kivinen2004}. 
The current instance can be inserted into the corresponding buffer if it is not filled yet. 
Or else a removal process should be employed first. 
All the budget kernel-based algorithms have to make two decisions: which instance should be removed, and the other is how to adjust the weights of the remaining instances~\citep{WangZ2010}. 
Many papers have proposed efficient methods for the removal process, such as the Forgetron~\citep{Dekel2008}, the Projectron~\citep{Orabona2009}, and the exact method~\citep{WangZ2010}. 

We consider the traditional stream oblivious policy, first in and first out(FIFO), in the removal step. 
Suppose the current example $(\bx_t,y_t)$ is positively labeled, we remove the instance, which appears in the buffer first, if the positive buffer is full.
Let $\bx_j, j = \arg\min I_t^+$ be the removed instance in $\mathcal K_t^+$. 
To avoid information loss, we should assign its weight to the remaining vectors. 
We choose one of them to be updated for simplicity, i.e. $r \in I_{t+1}^+ = I_t^+ \cup \{t\} \backslash j$. 
Then the update rule in $t$-th trial is
\ben
&& f_{t+1}^{'} = f_{t} - \eta_{t} \nabla L_{t}(f_{t}); \\
&& f_{t+1} = f_{t+1}^{'} - \alpha_{j,t+1}^{'} K(.,\bx_j) + \Delta \alpha_{r,t+1} K(.,\bx_r).
\enn
We choose $r, \Delta \alpha_{r,t}$ by solving the following optimization problem:
\begin{align*}
r^\star, \Delta \alpha_{r,t+1}^\star = \arg \min_{r, \Delta \alpha_{r,t+1} \in \mathbb R} \|f_{t+1} - f_{t+1}^{'} \|^2 = \arg\min_{r, \Delta \alpha_{r,t+1} \in R} \|\alpha_{j,t+1}^{'} k(.,\bx_j) - \Delta \alpha_{r,t+1} k(.,\bx_r) \|^2
\end{align*}
Then we have:
\begin{align}\label{eq:optimalR}
r^{\star} = \arg \min_{r}|k(\bx_r,\bx_j)| \qquad 
\Delta^{\star} \alpha_{r^{\star},t+1} = \frac{\alpha_{j,t+1}^{'} k(\bx_r^{\star},\bx_j)}{k(\bx_r^{\star},\bx_r^{\star})}
\end{align}
If the current example $(\bx_t,y_t)$ is negative labeled, there exists a similar result.
\begin{algorithm}
  \renewcommand{\algorithmicrequire}{\textbf{Input:}}
  \renewcommand{\algorithmicensure}{\textbf{Output:}}
  \caption{UpdateBuffer(${\mathbf z}_t, \mathcal K, B, I, \alpha_{t+1}^{'}$)}
  \label{alg-4}
  \begin{algorithmic}[1]
    \REQUIRE Given ${\mathbf z}_t, \mathcal K, B, I, \alpha_{t+1}^{'}$. \\ 
    \ENSURE Update the buffer.
    \IF {$|\mathcal K| < B$}
    \STATE  $\mathcal K = \mathcal K \cup \bx_t$, $I = I \cup t$, and $\alpha_{t+1} = \alpha_{t+1}^{'}$;
    \ELSE
    \STATE  $j = \arg \min I$;
    \STATE  Compute $r^{\star}$ and $\Delta^{\star} \alpha_{r^{\star},t+1}$ according to \eqref{eq:optimalR};
    \STATE  Update $\mathcal K = (\mathcal K \backslash x_i) \cup \bx_t $;
    \STATE  Update $ I = (I \backslash i) \cup t $;   
    \STATE  Update $\alpha_{t+1} = (\alpha_{i,t+1})$ as follow:
    \begin{align*}
    \alpha_{i,t+1} = 
    \begin{cases}
    \alpha_{r^{\star},t+1}^{'} + \Delta^{\star} \alpha_{r^{\star},t+1}, & \text{ if } i = r^{\star} \\ 
    0 & \text{ if } i = j \\
    \alpha_{i,t+1}^{'} & \text{ otherwise } 
    \end{cases}
    \end{align*}    
    \ENDIF
    \RETURN $\mathcal K$, $I$, and $\alpha_{t+1}$.
  \end{algorithmic}
\end{algorithm}

\subsection{Regret Analysis}

In this section, we will prove a regret bound of the online kernel AUC maximization algorithm based on the second-order surrogate loss $\psi_{M}$. 

\begin{lemma}\label{lemma-bound-f}
Suppose that for all $\bx \in \mathbb R^p$. 
Let $k$ be a symmertric positive semidefinite kernel and $k(\bx,\bx) \le 1$ for all $\bx$. 
Let the step size be $\eta_t = \frac{1}{\lambda t}$ and the initial classifier be $f_1 = 0$.
After running the online kernel AUC algorithm based on the second-order surrogate loss $\psi_{M}$ with an OGD update rule(Algorithm \ref{alg-okauc}), we have
\begin{align}
\|f_t\|_{\mathcal H} \le \frac{\sqrt2+\frac12}{\lambda}
\end{align}
\end{lemma}

The boundedness of the classifier norm established in Lemma \ref{lemma-bound-f} serves as a fundamental prerequisite for deriving the subsequent regret analysis. With this key property in hand, we now proceed to characterize the regret bound of the OKAUC-M algorithm under the ideal setting of infinite buffer size.

\begin{theorem}\label{The-regret-kernel-M}
Suppose that the assumptions of Lemma \ref{lemma-bound-f} hold. Then the regret bound of the algorithm OKAUC-M with infinite buffer size is given as 
\ben
Regret_T^{AUC}  \le  \frac{(2\sqrt2+1)^2}{2\lambda}(1+\ln T)
\enn
\end{theorem}

In contrast to KOIL, which employs a localized instantaneous AUC loss estimated from the k-nearest opposite instances via pairwise hinge loss, our OKAUC-SOSL framework adopts a globalized measure by evaluating the AUC loss over the entire comparison subset $S_t$. Furthermore, while KOIL updates the weights of comparison instances uniformly, OKAUC-SOSL adjusts weights adaptively based on the discrepancy between the predictor output and its mean $\mu_t$, leading to a more nuanced and data-dependent optimization process. Theoretically, OKAUC-M achieves an $\mathcal{O}(lnT)$ regret bound under the assumption of an infinite buffer, with the learning rate set as following \citet{Hazan2007}. However, for practical computational efficiency, a fixed-size buffer is implemented in actual deployments.

\section{Experiments}\label{sec7}

In this section, we conduct some experiments to evaluate the performance of OAUC-SOSL on several class-imbalanced benchmark datasets.

\subsection{Compared Algorithms}

We compare the OAUC-SOSL algorithm with several state-of-the-art online AUC maximization algorithms. Since our focus is on online learning, we do not include existing batch AUC methods in the comparison to ensure fairness. The algorithms evaluated in our experiments are as follows:\\
\begin{itemize}
    \item \textbf{Perceptron}: the Perceptron algorithm~\citep{Rosenblatt1958};
    \item \textbf{PA-I}: the Passive-Aggressive algorithm (PA-I)~\citep{Crammer2006};
    \item \textbf{OAM}: the online linear AUC maximization algorithm using gradient descent updating~\citep{Zhao2011};
    \item \textbf{OAUC-S}: the One-Pass AUC (OPAUC) algorithm from~\citep{Gao2016}, i.e., the proposed online AUC maximization algorithm based on the second-order surrogate loss $\psi_{S}$;
    \item \textbf{OL-UBAUC}: the online version of the Univariate Bound of Area Under ROC algorithm~\citep{Lyu2018};
    \item \textbf{OAUC-M$_c$}: the proposed online AUC maximization algorithm based on the second-order surrogate loss $\psi_{M}$ with a constant step size;
    \item \textbf{OAUC-M}: the proposed online AUC maximization algorithm based on the second-order surrogate loss $\psi_{M}$ with a predefined step size $\eta_t = \frac{1}{\lambda t}$;
    \item \textbf{OKAUC-S}: the proposed online kernel AUC maximization algorithm based on the second-order surrogate loss $\psi_{S}$;
    \item \textbf{OKAUC-M}: the proposed online kernel AUC maximization algorithm based on the second-order surrogate loss $\psi_{M}$ with a predefined step size $\eta_t = \frac{1}{\lambda t}$;
    \item \textbf{KOIL}: the online non-linear AUC maximization algorithm proposed in~\citet{Hu2018}.
\end{itemize}

\subsection{Datasets}

We conduct our experiments on several benchmark datasets.  
All datasets can be downloaded from the LIBSVM website\footnote{\url{https://www.csie.ntu.edu.tw/~cjlin/libsvm/}} or the UCI Machine Learning Repository\footnote{\url{http://archive.ics.uci.edu/ml/datasets.html}}.  
The original multi-class datasets are converted into binary classification tasks.  
Since these algorithms are designed for class-imbalanced binary datasets, without loss of generality, we set the minority class as the positive class and the majority class as the negative class.  
We define the imbalance ratio (IR) as the ratio of the number of negative samples to the number of positive samples.  
In our experiments, the imbalance ratio of the datasets ranges from 1 to 50.  
Table~\ref{tab:datasets} presents the details of these datasets.  
All features are scaled to the range $[-1,1]$.

\begin{table}[h!]
\caption{Details of the datasets} 
\centering
\begin{tabular}{|c|c|c|c||c|c|c|c|}
\hline 
Dataset  &     Instances  &     Features  &     IR  &     dataset  &     Instances  &     Features  &     IR \\   \hline
splice  &     1000  &     60  &     1.07 & australian  &     690  &     14  &     1.25 \\   \hline
heart  &     270  &     13  &     1.25  & svmguide1  &     7089  &     4  &     1.29 \\   \hline
ionosphere  &     351  &     34  &     1.79  & fourclass  &     862  &     2  &     1.81 \\   \hline
magic04  &     19020  &     10  &     1.84  & diabetes  &     768  &     8  &     1.87 \\   \hline
german  &     1000  &     24  &     2.33  & vehicle  &     846  &     18  &     2.90 \\   \hline
svmguide3  &     1284  &     21  &     3.34  & segment  &     2310  &     19  &     6.00 \\   \hline
svmguide2  &     391  &     20  &     6.38  & satimage  &     6435  &     36  &     9.28 \\   \hline
vowel  &     990  &     10  &     10.00  &  letter  &     15000  &     16  &     26.88 \\   \hline
shuttle  &     43500  &     9  &     44.89 & poker  &     25010  &     10  &     47.75 \\   \hline
\end{tabular}
\label{tab:datasets}
\end{table}

\subsection{Cross-Validation}
\begin{sidewaystable}[htbp]
\caption{Evaluation of average AUC performance of linear algorithms \%}
\centering 
\centering
\begin{tabular}{|c|c|c|c|c|c|c|c|c|}
\hline\hline 
Datasets     &   Perceptron           &   $PA-I$             &  OAM                 &  OL-UBAUC             & OAUC-S               & OAUC-M$_c$              & OAUC-M       \\  \hline
fourclass    &   79.81  $\pm$ 2.05    &   82.81  $\pm$ 1.24  &   83.10  $\pm$ 1.16  &   82.63  $\pm$ 1.26  &   83.17  $\pm$ 1.12  &   83.14  $\pm$ 1.14  &{\bf 83.3 $\pm$ 1.14}\\  \hline
diabetes     &   81.13  $\pm$ 1.48    &   83.15  $\pm$ 0.88  &   82.55  $\pm$ 0.95  &   82.80  $\pm$ 1.16  &   83.16  $\pm$ 1.00  &{\bf 83.36$\pm$ 0.96} &   83.25  $\pm$ 0.92 \\  \hline
german       &   73.68  $\pm$ 2.47    &   79.11  $\pm$ 1.39  &   76.74  $\pm$ 1.68  &   75.32  $\pm$ 2.37  &   80.18  $\pm$ 1.33  &   80.18  $\pm$ 1.32  &{\bf 80.19$\pm$ 1.24}\\  \hline
splice       &   83.76  $\pm$ 1.08    &   89.00  $\pm$ 0.91  &   87.03  $\pm$ 1.02  &   87.76  $\pm$ 1.08  &   89.81  $\pm$ 0.59  &   90.24  $\pm$ 0.59  &{\bf 90.37$\pm$ 0.58}\\  \hline
svmguide1    &   91.29  $\pm$ 0.71    &   92.60  $\pm$ 0.44  &   98.92  $\pm$ 0.12  &{\bf 99.01$\pm$ 0.08} &   97.34  $\pm$ 0.19  &   98.54  $\pm$ 0.10  &   98.54  $\pm$ 0.10 \\  \hline
magic04      &   71.19  $\pm$ 0.54    &   74.29  $\pm$ 0.46  &   72.26  $\pm$ 0.89  &   73.73  $\pm$ 0.49  &   75.34  $\pm$ 0.29  &{\bf 75.70$\pm$ 0.27} &   75.61  $\pm$ 0.23 \\  \hline
australian   &   89.96  $\pm$ 1.99    &   92.25  $\pm$ 1.14  &   92.28  $\pm$ 1.15  &   90.87  $\pm$ 1.98  &   92.65  $\pm$ 1.03  &{\bf 92.66$\pm$ 1.00} &   92.62  $\pm$ 0.93 \\  \hline
heart        &   87.19  $\pm$ 2.46    &   90.76  $\pm$ 1.32  &   91.12  $\pm$ 1.47  &   89.72  $\pm$ 1.90  &   91.21  $\pm$ 1.26  &{\bf 91.29$\pm$ 1.26} &   91.18  $\pm$ 1.33 \\  \hline
ionosphere   &   88.72  $\pm$ 2.79    &   91.96  $\pm$ 1.32  &   94.05  $\pm$ 1.52  &   91.63  $\pm$ 1.90  &   93.28  $\pm$ 1.26  &{\bf 93.88$\pm$ 1.08} &   93.47  $\pm$ 1.05 \\  \hline
svmguide3    &   71.69  $\pm$ 1.85    &   71.41  $\pm$ 1.44  &   74.85  $\pm$ 2.30  &   74.81  $\pm$ 1.18  &   76.07  $\pm$ 1.08  &{\bf 76.23$\pm$ 1.05} &   75.65  $\pm$ 0.83 \\  \hline
svmguide2    &   87.19  $\pm$ 1.86    &   65.16  $\pm$ 3.36  &{\bf 88.90$\pm$ 1.49} &   84.51  $\pm$ 2.11  &   87.66  $\pm$ 1.77  &   87.57  $\pm$ 1.50  &   88.17  $\pm$ 1.49 \\  \hline
vehicle      &   74.11  $\pm$ 2.34    &   73.48  $\pm$ 2.77  &   81.30  $\pm$ 2.65  &   78.21  $\pm$ 2.53  &   82.81  $\pm$ 1.92  &{\bf 83.21$\pm$ 1.88} &   82.21  $\pm$ 1.95 \\  \hline
vowel        &   86.88  $\pm$ 1.81    &   83.40  $\pm$ 4.90  &   91.32  $\pm$ 1.17  &   90.90  $\pm$ 1.38  &   91.53  $\pm$ 1.37  &{\bf 91.72$\pm$ 1.37} &   91.57  $\pm$ 1.38 \\  \hline
segment      &   85.47  $\pm$ 1.42    &   70.35  $\pm$ 4.94  &   89.94  $\pm$ 0.95  &   88.99  $\pm$ 1.03  &   87.22  $\pm$ 0.70  &   88.76  $\pm$ 0.60  &   88.52  $\pm$ 0.73 \\  \hline
satimage     &   70.81  $\pm$ 1.62    &   73.72  $\pm$ 1.39  &   75.08  $\pm$ 0.78  &   73.68  $\pm$ 1.52  &   75.35  $\pm$ 0.76  &{\bf 76.14$\pm$ 0.76} &   75.69  $\pm$ 0.74 \\  \hline
poker        &   52.28  $\pm$ 0.58    &   52.14  $\pm$ 0.42  &   50.47  $\pm$ 0.33  &   52.15  $\pm$ 0.65  &   51.84  $\pm$ 0.49  &   52.53  $\pm$ 0.47  &{\bf 53.57$\pm$ 0.39}\\  \hline
letter       &   75.28  $\pm$ 1.74    &   70.93  $\pm$ 3.73  &   80.53  $\pm$ 0.64  &   79.29  $\pm$ 0.65  &   82.98  $\pm$ 0.44  &{\bf 83.33$\pm$ 0.48} &   83.29  $\pm$ 0.44 \\  \hline
shuttle      &   92.08  $\pm$ 0.50    &   75.92  $\pm$ 0.47  &{\bf 96.53$\pm$ 0.32} &   93.20  $\pm$ 0.43  &   95.66  $\pm$ 0.20  &   94.99  $\pm$ 0.38  &   95.19  $\pm$ 0.46 \\  \hline
\end{tabular}
\label{tab:linear_auc}
\end{sidewaystable}
For each dataset, we randomly split it into 5 folds.  
Four folds are used for training, and the remaining fold is used for testing.  
To reduce the variance of the results, we generate 4 independent 5-fold partitions for each dataset, resulting in a total of 20 runs per dataset.  
The reported AUC values are averaged over these 20 runs.  
Five-fold cross-validation is performed on the training sets to select the regularization parameter $\lambda \in 2^{[-10,10]}$ for all algorithms.  
The regret bounds in OPAUC (OAUC-S), OAUC-M, and OL-UBAUC suggest the optimal learning rate; however, in OPAUC (OAUC-S), the optimal learning rate depends on the optimal loss $L^{\star}$, which is unknown in advance.  
To ensure fair comparisons, we determine the learning rate $\eta_t \in 2^{[-10,10]}$ for all algorithms via cross-validation (except for UBAUC, for computational efficiency).  
We also run OAUC-M with a learning rate $\eta_t = \frac{1}{\lambda t}$.  
In OAM, the sizes of the positive and negative buffers are both fixed at 100, as recommended in~\citet{Zhao2011}.  
In the kernel-based methods, the positive and negative support buffer sizes are also fixed at 100.  
The number of nearest instances is set to 10, following~\citep{Hu2018}.  
For the kernel methods, we use the Gaussian kernel  
\[
k(\bx,\bx') = \exp\left(-\frac{\|\bx-\bx'\|^2}{\sigma^2}\right),
\]  
and the kernel width parameter $\sigma$ is selected via five-fold cross-validation with $\sigma \in 2^{[-10,10]}$. 
Due to memory limitations, in kernel-based experiments, if the training set size exceeds 10,000, we uniformly sample 10,000 training examples at random (without replacement) from the full training set.

\subsection{Performance Evaluation}
The average AUC performances of the online linear AUC optimization algorithms and the online kernel-based AUC optimization algorithms are presented in Table~\ref{tab:linear_auc} and Table~\ref{tab:kernel_auc}, respectively.  
From the experimental results, we make the following observations:  
\begin{enumerate}
    \item The proposed OAUC-M and OAUC-M$_c$ algorithms, based on the second-order surrogate loss $\psi_{M}$, achieve competitive or superior performance compared with other state-of-the-art online AUC optimization algorithms.
    \item On most datasets, kernel-based algorithms achieve better AUC performance, demonstrating the effectiveness of these methods in real-world AUC learning problems. In particular, compared with the KOIL algorithm, kernel algorithms based on the second-order surrogate loss $\psi_{M}$ are more stable and effective.
\end{enumerate}

\begin{table}[h!]
\caption{Evaluation of average AUC performance of kernel-based algorithms \%}
\centering
\begin{tabular}{|c|c|c|c|}
\hline\hline 
Datasets   &     KOIL              &     OKAUC-M            &  OKAUC-S             \\  \hline
fourclass  &     98.96 $\pm$ 1.27  &     99.94 $\pm$ 0.08   &{\bf 99.95 $\pm$ 0.03}\\  \hline
diabetes   &     82.68 $\pm$ 1.29  &{\bf 84.55 $\pm$ 1.28}  &     84.43 $\pm$ 1.28 \\  \hline
german     &{\bf 81.56 $\pm$ 1.24} &     80.90 $\pm$ 1.43   &     81.20 $\pm$ 1.27 \\  \hline
splice     &     90.75 $\pm$ 2.05  &{\bf 91.32 $\pm$ 1.21}  &     90.92 $\pm$ 1.68 \\  \hline
svmguide1  &     98.94 $\pm$ 0.13  &     99.04 $\pm$ 0.08   &{\bf 99.10 $\pm$ 0.07}\\  \hline
magic04    &     77.27 $\pm$ 0.78  &     77.34 $\pm$ 1.20   &{\bf 77.69 $\pm$ 0.70}\\  \hline
australian &{\bf 92.16 $\pm$ 1.11} &     91.76 $\pm$ 1.00   &     91.88 $\pm$ 0.85 \\  \hline
heart      &     94.48 $\pm$ 1.80  &{\bf 97.62 $\pm$ 0.46 } &     96.49 $\pm$ 2.39 \\  \hline
ionosphere &     97.52 $\pm$ 1.51  &{\bf 99.18 $\pm$ 0.30}  &     99.12 $\pm$ 0.40 \\  \hline
svmguide3  &     76.04 $\pm$ 2.59  &{\bf 78.53 $\pm$ 1.34}  &     77.32 $\pm$ 2.26 \\  \hline
svmguide2  &     89.20 $\pm$ 3.28  &{\bf 96.87 $\pm$ 1.00}  &     95.64 $\pm$ 1.43 \\  \hline
vehicle    &     83.59 $\pm$ 1.69  &{\bf 86.78 $\pm$ 1.22}  &     85.98 $\pm$ 1.21 \\  \hline
vowel      &     98.60 $\pm$ 0.87  &{\bf 99.99 $\pm$ 0.01}  &     99.91 $\pm$ 0.19 \\  \hline
segment    &     97.82 $\pm$ 0.54  &{\bf 98.53 $\pm$ 0.20}  &     98.51 $\pm$ 0.23 \\  \hline
satimage   &     92.19 $\pm$ 1.03  &{\bf 94.16 $\pm$ 1.22}  &     93.81 $\pm$ 1.55 \\  \hline
poker      &     84.90 $\pm$ 0.84  &     85.17 $\pm$ 1.11   &{\bf 85.22 $\pm$ 1.22}\\  \hline
letter     &     87.56 $\pm$ 5.21  &     95.90 $\pm$ 0.50   &{\bf 96.09 $\pm$ 0.51}\\  \hline
shuttle    &     98.45 $\pm$ 0.45  &     99.59 $\pm$ 0.13   &{\bf 99.60 $\pm$ 0.13} \\ \hline
\end{tabular}
\label{tab:kernel_auc}
\end{table}

\section{Conclusion}\label{sec8}

This paper has investigated the problem of online AUC maximization for class-imbalanced binary classification. Unlike conventional approaches that rely on instant-wise pairwise convex surrogate losses, we have introduced a novel paradigm that directly substitutes the entire AUC risk (aggregated pairwise loss) using statistical moments. 

A central contribution of this work is the derivation of the second-order surrogate loss $\psi_M$ from a robust optimization perspective under moment constraints. By framing the worst-case hinge loss over all distributions sharing the same first and second-order moments, we obtain a tractable and theoretically justified upper bound that depends only on the mean and covariance of the data. This formulation not only facilitates efficient online learning but also opens promising avenues for extending the approach to distributionally robust optimization (DRO) settings, where more general uncertainty sets or divergence measures could be incorporated to enhance model robustness against distribution shifts\citep{Rahimian2019,Kuhn2025}.

\textbf{Future work:} Despite its advantages, the proposed second-order surrogate loss framework exhibits certain limitations that warrant further investigation. The current formulation requires explicit computation and storage of the covariance matrix, which becomes prohibitive in high-dimensional settings. This limitation restricts its direct applicability to large-scale feature spaces or deep learning architectures where dimensionality is substantial. Inspired by the SOLAM algorithm, a promising direction for future research involves reformulating the hinge-based AUC optimization problem as a stochastic saddle point problem based on $\phi_M$. Such a transformation could potentially eliminate the need for explicit covariance computation while preserving the statistical benefits of the second-order surrogate loss. By adopting a primal-dual optimization framework similar to SOLAM, we may achieve linear space and time complexity in the stochastic setting, thereby extending the applicability of our method to high-dimensional and deep learning scenarios without sacrificing theoretical guarantees.

In conclusion, the second-order surrogate loss proposed in this paper offers a novel and efficient strategy for online AUC maximization. We believe this work not only advances the state of online learning in imbalanced classification but also provides a foundation for future research in robust and large-scale pairwise learning problems.

\section*{Appendix}
\subsection*{Proof of Lemma \ref{lemma-Bound-H}}
\begin{proof}
Let $\mathcal{C}_k$ denote the subset of $\mathcal{C}$ containing exactly $k$ components no greater than 1, i.e.,
$$\mathcal{C}_k = \{\mathbf{z} \in \mathcal{C}, \text{ where } \sum_{i=1}^n \mathbb I(z_i \le 1) = k\}, 0 \le k \le n.$$
Since $\mathcal{C} = \bigcup_{k=0}^n \mathcal{C}_k$, an upper bound for $\ell(\mathbf{z})$ on $\mathcal{C}$ can be obtained by bounding $\ell(\mathbf{z})$ on each $\mathcal{C}_k$. Define the restricted loss function on $\mathcal{C}_k$ as
\ben
\ell_k(\mathbf{z}) = \frac 1n \sum_{\mathbf{z} \in C_k} \max(0,1-z_i).
\enn
Then 
\ben
\max_{\mathbf{z} \in \mathcal{C} } \ell(\mathbf{z}) = \max_{k \in [n]}{\max_{\mathbf{z} \in \mathcal{C}_k }\ell_k(\mathbf{z}) } = \max_{k \in [n]} g(k)
\enn
The proof is structured into four parts: In Part 1, we derive the expression for $g(k) = \max_{\mathbf{z} \in \mathcal{C}_k} \ell_k(\mathbf{z})$; in Parts 2 and 3, we maximize $g(k)$ for $\mu < 1$ and $\mu \ge 1$, respectively; Part 4 concludes the proof.

\noindent \textbf{Part 1.}
For $k = 0$, we have $\ell_0(\mathbf{z}) = 0$ for all $\mathbf{z} \in \mathcal{C}_0$; for $k = n$, $\ell_n(\mathbf{z}) = 1 - \mu$ for all $\mathbf{z} \in \mathcal{C}_n$.

Now consider $0 < k < n$. Without loss of generality, assume that for $\mathbf{z} \in \mathcal{C}_k$, the first $k$ components satisfy $z_i \le 1$ for $i = 1, \dots, k$, and the remaining $z_j > 1$ for $j = k+1, \dots, n$. Then,
 $$\ell_k(\mathbf{z}) = \frac{1}{n}\sum_{i=1}^n \max(0,1-z_i) = \frac{1}{n} \sum_{i=1}^k (1-z_i).$$
Maximizing $\ell_k(\mathbf{z})$ is equivalent to minimizing $\sum_{i=1}^k z_i$ subject to the constraints:
\ben
\min   && \sum_{i=1}^k z_i \\
\text{s.t.} && \sum_{i=1}^n (z_i-\mu) = 0 \\ 
            && \sum_{i=1}^n (z_i-\mu)^2 = n\sigma^2 \\
      && z_i \le 1 ,i=1,2,\cdots,k  \\ 
      && z_j > 1 ,j=k+1,\cdots,n.
\enn
The Lagrange function is:
\ben
L(\mathbf{z};\lambda_1,\lambda_2,\lambda_3^i,\lambda_4^j) =&& \sum_{i=1}^k z_i + \lambda_1 \sum_{i=1}^n (z_i -\mu) +\lambda_2 (\sum_{i=1}^n (z_i-\mu)^2 - n\sigma^2) \\
 && + \sum_{i=1}^k \lambda_3^{i}(z_i-1) + \sum_{j=k+1}^n \lambda_4^{j}(1-z_j).
\enn
The KKT conditions are:
\begin{align}
& \frac{\partial L}{\partial z_i} = 1 + \lambda_1 + 2\lambda_2(z_i- \mu) + \lambda_3^i = 0 ,\qquad i = 1,2,...,k  \label{eq:kkt1} \ \\
& \frac{\partial L}{\partial z_j} = \lambda_1 + 2 \lambda_2(z_j- \mu) - \lambda_4^j = 0 ,\qquad j = k+1,k+2,...,n  \label{eq:kkt2} \\\
& \sum_{i=1}^n (z_i-\mu) = 0 ,\  \sum_{i=1}^n (z_i-\mu)^2 = n\sigma^2  \label{eq:kkt3} \ \\
& \lambda_3^i(z_i-1) = 0 , \  z_i \le 1, \ \lambda_3^{i} \ge 0 ,\qquad i = 1,2,\cdots,k  \label{eq:kkt4} \ \\
& \lambda_4^j(1 - z_j) = 0 , \  z_j  > 1, \  \lambda_4^{j}  = 0 , \qquad j = k+1,k+2,\cdots,n.  \label{eq:kkt5} \
\end{align}

We claim $\lambda_2 \neq 0$. Suppose $\lambda_2 = 0$, then from \eqref{eq:kkt2} and \eqref{eq:kkt5}, there are $\lambda_1 = 0$ and $\lambda_4^{j}  = 0 $. Then \eqref{eq:kkt1} gives $1 + \lambda_3^i > 0$, contradicting the stationarity condition. Hence, $\lambda_2 \neq 0$.

From \eqref{eq:kkt1} and \eqref{eq:kkt2}:
\begin{align*}
 && z_i- \mu = -\frac1{2\lambda_2}(1+\lambda_1+\lambda_3^i)  ,\qquad i = 1,2,\cdots ,k,\\
 && z_j- \mu = -\frac1{2\lambda_2}\lambda_1 ,\qquad j = k+1,k+2,...,n.
\end{align*}
Substituting into \eqref{eq:kkt3}:

\begin{align}
\sum_{i=1}^n (z_i-\mu) = -\frac1{2\lambda_2}(k+n\lambda_1+\sum_{i=1}^k \lambda_3^i) = 0, \label{KKT-sum1}
\end{align}
 and
\begin{align}
\sum_{i=1}^n (z_i-\mu)^2 = \frac1{4\lambda_2^2}[\sum_{i=1}^k (1+\lambda_1+\lambda_3^i)^2 +  \sum_{i=1+k}^n\lambda_1^2] = n\sigma^2. \label{KKT-sum2}
\end{align}
Then the KKT condition can be rewritten as
 \begin{align}\label{KKT2}
 \begin{split}
& k+n\lambda_1+\sum_{i=1}^k \lambda_3^i -\sum_{i=k+1}^n \lambda_4^j = 0 \\
& \sum_{i=1}^k (1+\lambda_1+\lambda_3^i)^2 +  \sum_{i=1+k}^n(\lambda_1-\lambda_4^j)^2 = 4\lambda_2^2 n \sigma^2 \\
& z_i = \mu - \frac1{2\lambda_2}(1+\lambda_1+\lambda_3^i) \le 1   ,\qquad i = 1,2,...,k \\
& z_j = \mu - \frac1{2\lambda_2}(\lambda_1-\lambda_4^j) \ge 1   ,\qquad j = k+1,k+2,...,n \\
& \lambda_3^i(z_i-1) = 0 , \  z_i \le 1, \ \lambda_3^{i} \ge 0 ,\qquad i = 1,2,\cdots,k\\
& \lambda_4^j(1 - z_j) = 0 , \  z_j  > 1, \  \lambda_4^{j}  = 0 , \qquad j = k+1,k+2,\cdots,n.
\end{split}
\end{align}
Because $\lambda_4^{j}  = 0$ for all $j = k+1,k+2,\cdots,n$, there is $z_j = q_2 > 1$. 

If there exist $i_1, i_2$ such that $\lambda_3^{i_1} = 0$ and $\lambda_3^{i_2} > 0$, then:
$$z_{i_1} = \mu -\frac1{2\lambda_2}(1+\lambda_1+\lambda_3^{i_1}) =  \mu -\frac1{2\lambda_2}(1+\lambda_1) \le 1,$$
and 
$$z_{i_2} = \mu -\frac1{2\lambda_2}(1+\lambda_1+\lambda_3^{i_2}) = z_{i_1} - \frac1{2\lambda_2}\lambda_3^{i_2} = 1.$$

Then $\frac{1}{2\lambda_2} \lambda_3^{i_2} \le 0$, and since $\lambda_3^{i_2} \ge 0$, we must have $\lambda_2 < 0$. But then for any $j$, $z_j = z_{i_1} + \frac{1}{2\lambda_2} < 1$, contradicting $z_j > 1$. Thus, all $\lambda_3^i$ are either all zero or all nonzero.

\noindent \textbf{Case 1:} $\lambda_3^i = 0$ for all $i = 1, \dots, k$.
According to \eqref{eq:kkt4}, We have $z_1=z_2=\cdots=z_k=q_1$, and $q_1 = \mu -\frac1{2\lambda_2}(1+\lambda_1) = q_2 - \frac1{2\lambda_2} < q_2$. So $\lambda_2>0$. 
From \eqref{KKT-sum1} and \eqref{KKT-sum2},  we have
\ben
&& k + n \lambda_1 =0, \\
&& k(1+\lambda_1)^2 + (n-k)\lambda_1^2 = 4\lambda_2^2 n \sigma^2.
\enn
Solving this equations, we obtain $\lambda_2 = \frac{\sqrt{(n-k)k}}{2n\sigma}$ and $\lambda_1 =-\frac kn$. $q_1,q_2$ are given as 
\ben
&& q_1 = \mu - \frac1{2\lambda_2}(1+\lambda_1)  = \mu - \sqrt{\frac{n-k}{k}}\sigma \le 1, \\
&& q_2 = \mu - \frac1{2\lambda_2}\lambda_1  = \mu + \sqrt{\frac{k}{n-k}}\sigma > 1.
\enn
So $k$ should satisfy the following condition:
\ben
\mu \ge 1, 0 < k \le \frac{n}{1+v^2}; \  \mu < 1, \frac{n v^2}{1+v^2} < k < n. 
\enn

\noindent \textbf{Case 2:} $\lambda_3^i \neq 0$ for all $i = 1, \dots, k$.

Then $z_i = 1$ for $i = 1, \dots, k$, and from the constraints, there is 
\ben
&& k(1-\mu)+(n-k)(q_2-\mu) = 0, \\
&& k(1-\mu)^2+(n-k)(q_2-\mu)^2 = n \sigma^2.
\enn
Solving these equations, we have $q_2 =  \mu + \sqrt{\frac{k}{n-k}}\sigma $, where $\mu \ge 1$ and $k= \frac{n}{1+v^2}$. 
Also, we have $q_1 = \mu - \sqrt{\frac{n-k}{k}}\sigma =1$.

In summary, we have $z_i = q_1 = \mu - \sqrt{\frac{n-k}{k}}\sigma $ and 
 \ben
g(k) &=&  \frac kn - \frac 1n \sum_{i=1}^k z_i =  \frac kn\left(1-\mu + \sqrt{\frac{n-k}{k}}\sigma\right) = \frac 1n(k(1-\mu) + \sqrt{k(n-k)}\sigma). 
 \enn
If $\mu \ge 1$, $k$ can be zero, and its loss $ g(0) = 0$. If $\mu < 1$, $k$ can be $n$, and its loss $g(n) = 1-\mu$. Then the upper bound of $\ell_k(\mathbf{z})$ on the subset $\mathcal{C}_k$ is
\ben
g(k) =  \frac 1n(k(1-\mu) + \sqrt{k(n-k)}\sigma),
 \enn
 and $k$ satisifies 
\ben
\mu \ge 1, 0 \le k \le \frac{n}{1+v^2}; \  \mu < 1, \frac{nv^2}{1+v^2} < k \le n. 
\enn. 

\noindent \textbf{Part 2.}

In this part we optimize $g(k)$ when $\mu < 1$. Taking derivation of k, we have
\ben
g^{'}(k)= \frac 1n(1-\mu + \frac{n-2k}{2\sqrt{(n-k)k}}\sigma),
\enn
and 
\ben
g^{''}(k) = -\frac{\sigma}{2n} \frac{(4k(n-k)+(n-2k)^2)}{(k(n-k))^{\frac32}}= -\frac{\sigma}{2} \frac{n}{(k(n-k))^{\frac32}}.
\enn
Then $ g(k)$ is concave on $0 \le k \le n$ and has a only one optimal maximization point. Suppose the optimal point is $k^{\star}$ and $k^{\star}$ is not need to be an integer. Notice that $\lim_{k \to 0}g(k) = \infty$ and $\lim_{k \to n}g(k) = -\infty$. According to the first-order optimal condition, $g^{'}(k^\star) = 0$. Then $k^{\star}$ can be solved by the following equation\\
\ben
\frac{1-\mu}{\sigma} = \frac12  \frac{(2k-n)}{\sqrt{(n-k)k}}.
\enn
Squaring the both sides of the equality and rearranging the equation, we have
\ben
4(1+v^2)k^2 - 4n(1+v^2)k +n^2 = 0 \text{ and } \  k > \frac n2.
\enn
Then we get the value of $k^{\star}$
\ben
k^{\star} &=& \frac{4n(1+v^2)+\sqrt{16n^2(1+v^2)^2-16(1+v^2)n^2}}{8(1+v^2)}=\frac n2 + \frac n2 \sqrt{\frac{v^2}{1+v^2}} = \frac{n}{2}\left(1+\frac{v}{\sqrt{1+v^2}}\right).
 \enn
When $\mu < 1$ , $k^{\star} > \frac{n v^2}{1+v^2}$, otherwise, if $k^{\star} \le \frac{n v^2}{1+v^2}$, that is $1 + \sqrt{\frac{v^2}{1+v^2}} \le \frac{2 v^2}{1+v^2}$, denote by $t=\sqrt{\frac{v^2}{1+v^2}}$. The above equation translate to  $2t^2-t-1=(2t+1)(t-1) \ge 0$. but it is not true when $t \in [0,1)$. So $k^{\star} > \frac{n v^2}{1+v^2}$. Pluging $k^{\star}$ into $g(k)$, we have,
 \ben
g(k^\star) &=& \frac 1n(k^{\star}(1-\mu) + \sqrt{k^{\star}(n-k^{\star})\sigma})\\
&=& \frac{1}{2}\left(1+\frac{v}{\sqrt{1+v^2}}\right)(1-\mu) + \frac12 \sqrt{(1+\frac{v}{\sqrt{1+v^2}})(1-\frac{v}{\sqrt{1+v^2}})} \sigma  \\
&=& \frac12 \left(1+ \frac{v}{\sqrt{1+v^2}} \right)(1-\mu)+\frac12\sqrt{\frac{1}{1+v^2}} \sigma \\
&=& \Phi_M(v)(1-\mu) + \phi_M(v)\sigma.
 \enn

\noindent \textbf{Part 3.}

Similarly, we optimize $g(k)$ when $\mu \ge 1$ in this part. And we have,
\ben
g(k^\star) &=& \Phi_M(v)(1-\mu) + \phi_M(v)\sigma.
 \enn

\noindent \textbf{Part 4.}

Then $\ell(\mathbf{z})$ is upper bounded by $\psi_{M}(\mu,\sigma)$, where 
\ben
\psi_{M}(\mu,\sigma) & = & \Phi_M(v)(1-\mu) + \phi_M(v)\sigma \\
& = & \frac12 \left(1+ \frac{v}{\sqrt{1+v^2}} \right)(1-\mu)+\frac12\sqrt{\frac{1}{1+v^2}}\sigma  \\
& = & \frac12 \left(1+\frac{1-\mu}{\sigma}\sqrt{\frac{1}{1+(\frac{1-\mu}{\sigma})^2}}\right)(1-\mu)+\frac12\sqrt{\frac{1}{1+(\frac{1-\mu}{\sigma})^2}}\sigma  \\
& = & \frac 12 (1-\mu) + \frac 12 \sqrt{(1-\mu)^2+ \sigma^2}.
\enn
\end{proof}

\subsection*{Proof of Lemma \ref{lemm-pro-M}}
\begin{proof}\\
\noindent
Denote by $\hat \bx = y(\bx-\bar \bx),\sigma = \sqrt{\bw^\top\Sigma \bw}$. Following the definition of $\Psi_M$, we have 
\ben
\Psi_{M}(\bw) =  \Phi_{M}\left(v \right)(1-\bw^\top \hat \bx) + \phi_M\left(v \right) \sigma.
\enn
\begin{enumerate}
    \item Taking derivative to the function $\Psi_{M}(\bw)$, we have 
\ben
\nabla \Psi_{M}(\bw) &&= -\hat \bx \Phi_M(v)+(1-\bw^\top \hat \bx) \Phi_M'(v) \nabla v + (\nabla \sigma) \phi_M(v)  + \sigma \phi_M'(v) \nabla v\\
&&= -\hat \bx \Phi_M(v)+(\nabla \sigma) \phi_M(v) + \sigma(v\Phi_M'(v)+\phi_M'(v))\nabla v.
\enn
It then follows that
\ben
\nabla \sigma &= \nabla \sqrt{\bw^\top\Sigma \bw}= (\bw^\top\Sigma \bw)^{-\frac12} \Sigma \bw=\frac{1}{\sigma}\Sigma \bw ,
\enn
and
\ben
v\Phi_M'(v)+\phi_M'(v)=\half v (1+v^2)^{-\frac32} - \half v (1+v^2)^{-\frac32} = 0.
\enn
Then we have
\ben
\nabla \Psi_{M}(\bw) = -\hat x \Phi_M(v)+(\nabla \sigma) \phi_M(v) = - \Phi_M(v) y(x-\bar x)  + \phi_M(v) \frac{\Sigma \bw}{\sqrt{\bw^\top\Sigma \bw}}.
\enn

\noindent
\item Taking derivative to the function $\nabla \Psi_{M}(w)$
\ben
\nabla^2 \Psi_{M}(\bw) &&= -\frac{\partial \Phi_M(v)}{\partial \bw} \hat x^\top + \frac{\partial \phi_M(v)/\sigma }{\partial \bw} (\Sigma \bw )^\top +  \frac{\phi(v)}{\sigma} \frac{\partial \Sigma \bw }{\partial \bw}\\
&&= -\Phi_M'(v) \nabla v\hat \bx^\top + \frac{1}{\sigma^2}\left(\sigma \phi_M'(v)\nabla v - \frac{\phi_M(v)}{\sigma} \Sigma \bw\right)(\Sigma \bw)^\top + \frac{\phi_M(v)}{\sigma} \Sigma \\
&&= \frac{\phi_M(v)}{\sigma} \left(\Sigma - \frac{(\Sigma \bw)(\Sigma \bw)^\top}{\sigma^2}\right) + \frac{\nabla v }{\sigma} \left(-\sigma\Phi_M'(v)\hat \bx^\top-\Phi_M'(v)v(\Sigma \bw)^\top \right) \\
&&= \frac{\phi_M(v)}{\sigma} \left(\Sigma - \frac{(\Sigma \bw)(\Sigma \bw)^\top}{\sigma^2}\right) + \frac{\Phi_M'(v)}{\sigma} (\sigma \nabla v)(\sigma \nabla v)^\top.
\enn
Denote by $ M = \Sigma -\frac 1 {\sigma^2} {(\Sigma \bw)(\Sigma \bw)^\top}$. If  $M$ is a positive semidefinite matrix, then the hessian matrix $\nabla^2 \Psi(\bw)$ is also positive semidefinite. Next we show that $M$ is a positive semidefinite matrix. Given any $\bw_0 \in \mathbb R^p$,
\ben
\bw_0^\top M \bw_0 = \frac{1}{\sigma^2}\left(\sigma^2 \bw_0^\top\Sigma \bw_0 - \bw_0^\top(\Sigma \bw)(\Sigma \bw)^\top \bw_0\right).
\enn
Since $\Sigma$ is positive semidefinite, there exists an orthogonal matrix $P$ subject to $ \Sigma = P^\top\Lambda P$,where $\Lambda = diag(\lambda_i)_{p \times p}$. $\lambda_1 \ge \lambda_2 \ge \cdots \lambda_p \ge 0 $ are the eigenvalues of the covariance matrix and the i-th row of matrix $P$ is the corresponding eigenvector to the eigenvalue $\lambda_{i}$.

Denote by $\hat \bw = P\bw$ and $\hat \bw_0 = P \bw_0$. Then
\ben
\bw_0^\top M \bw_0 &&= \frac{1}{\sigma^2}\left(\sigma^2 \bw_0^\top\Sigma \bw_0 - \bw_0^\top(\Sigma \bw)(\Sigma \bw)^\top \bw_0\right) \\
&& = \frac{1}{\sigma^2}\left((\sum_{i=1}^{p}\lambda_i \hat \bw_i^2)(\sum_{i=1}^{p}\lambda_i \hat {\bw_0}_i^2) - (\sum_{i=1}^{p}\lambda_i \hat {\bw_0}_i\hat {\bw}_i)^2  \right) \\
&& \ge 0.
\enn
The last inequation can be deduced by the Cauchy-Schwartz inequality. Then the Hessian matrix $\nabla^2 \Psi_{M}(w)$ is positive semidefinite and $\Psi_{M}(\bw)$ is a convex function with respect to $w$.

\noindent
\item If $\forall \bx , \|\bx\| \le 1$, we have $\|\bx-\bar \bx\| \le 2$. And we have 
\ben
\left \|\frac{\Sigma \bw}{\sigma} \right \|^2 = \frac{\bw^\top\Sigma^\top\Sigma \bw}{ \bw^\top\Sigma \bw} = \frac{(P\bw)^\top \Lambda P P^\top \Lambda P \bw}{(P \bw)^\top \Lambda P \bw} = \frac{(P \bw)^\top \Lambda^2 P \bw}{(P \bw)^\top \Lambda P \bw} = \frac{\sum_{i=1}^p (\lambda_i \hat \bw_i)^2}{\sum_{i=1}^p \lambda_i \hat \bw_i^2}\le \frac{\sum_{i=1}^p \lambda_1 \lambda_i \hat \bw_i^2}{\sum_{i=1}^p \lambda_i \hat \bw_i^2} \le  \lambda_1.
\enn
In the light of the properties of eigenvalues, we have
\ben
\lambda_1 \le \sum_{i=1}^p \lambda_i = tr(\Sigma) = tr\left(\frac 1{n} \sum_{i=1}^n(\bx_i-\bar \bx)(\bx_i-\bar \bx)^\top \right) = tr \left(\frac 1{n} \sum_{i=1}^n(\bx_i-\bar \bx)^\top(\bx_i-\bar \bx) \right) \le 4
\enn
Then we have
\ben
\|\nabla \Psi_{M}(\bw)\| =  \left \|- \Phi_M(v) y(\bx-\bar \bx)  + \phi_M(v) \frac{\Sigma \bw}{\sigma} \right\| \le \Phi_M(v) \| \bx-\bar \bx\| + \phi_M(v) \left \|\frac{\Sigma \bw}{\sigma} \right \| \le 3
\enn
\end{enumerate}
\end{proof}

\subsection*{Proof of Theorem \ref{regretofOAUC2}}
\begin{proof}
For simplicity, denote by $\Psi_t(\bw)= \psi_{M}(y_t\bw^\top(\bx_t-\bar \bx_t);\bw^\top \Sigma_t \bw)$. 
According to the Lemma \ref{lemm-pro-M}, we have $\| \nabla \Psi_t(\bw) \| \le C = 3$. 
Using the triangle inequality, we have
\ben
\|\bw_{t+1}\| = \|(1-\eta_t \lambda)\bw_t - \eta_t \nabla \Psi_t(\bw_t))\| \le (1- \lambda \eta_t )\|\bw_t \|  + \eta_t C.
\enn
By expanding $\|\bw_{t+1}\|$ iteratively and setting $\eta_t = \frac{1}{\lambda t} $, we have
\ben
\|\bw_{t+1}\| && \le \eta_t C + (1-\lambda \eta_t)\eta_{t-1}C + (1-\lambda \eta_t)(1-\lambda \eta_{t-1})\eta_{t-2}C + \cdots \\ 
&&+ (1-\lambda \eta_t)(1-\lambda \eta_{t-1})\cdots (1-\lambda \eta_2)\eta_1C \\
&& = C\left(\frac{1}{\lambda t} + \frac{t-1}{t}\frac{1}{\lambda(t-1)}+\cdots + \frac{t-1}{t}\frac{t-2}{t-1}\cdots\frac{1}{2}\frac{1}{\lambda} \right)\\
&& \le \frac{C}{\lambda}.
\enn
Then we can upper bound the gradient of the t-th AUC loss,
\ben
\| \nabla L_t(\bw_t) \|^2  = \| \lambda \bw_t + \nabla \Psi_t(\bw_t)\|^2 \le 2(\lambda^2 \|\bw_t\|^2 + \|\nabla \Psi_t(\bw_t)\|^2) \le4C^2.
\enn
According to the update rule, we have
\ben
\| \bw_{t+1}-\bw^{\star} \|^2 = \|\bw_t-\eta_t \nabla L_t(\bw_t) -\bw^{\star} \|^2 = \|\bw_t-\bw^{\star}\|^2 - 2 \left< \bw_t-\bw^{\star},\eta_t \nabla L_t(\bw_t) \right> + \eta_t^2 \|\nabla L_t(\bw_t)\|^2.
\enn
By the $\lambda$-strongly convexity of $L_t(\bw_t)$,
\ben
L_t(\bw^{\star}) - L_t(\bw_t) \ge \left<\nabla L_t(\bw_t),\bw^{\star}-\bw_t \right> + \frac{\lambda}{2} \|\bw_t -  \bw^{\star}\|^2= - \left<\nabla L_t(\bw_t),\bw_t-\bw^{\star}\right> + \frac{\lambda}{2} \|\bw_t -  \bw^{\star}\|^2.
\enn
That is
\ben
2\eta_t(L_t(\bw_t)-L_t(\bw^{\star})) &&\le  \left<2\eta_t \nabla L_t(\bw_t),\bw_t-\bw^{\star}\right> - \eta_t \lambda\|\bw_t -  \bw^{\star}\|^2 \\
&& = \|\bw_t-\bw^{\star}\|^2- \| \bw_{t+1}-\bw^{\star} \|^2 + \eta_t^2 \|\nabla L_t(\bw_t)\|^2 - \eta_t \lambda\|\bw_t -  \bw^{\star}\|^2 \\
&& \le \|\bw_t-\bw^{\star}\|^2- \| \bw_{t+1}-\bw^{\star} \|^2 + 4\eta_t^2 C^2 - \eta_t \lambda\|\bw_t -  \bw^{\star}\|^2.
\enn
Summing over $t=1,2,\cdots,T$ and rearranging, we obtain
\ben
\text{Regret}_{T}^{AUC} && \le \sum_{i=1}^\top \frac{1}{2\eta_t}(\|\bw_t-\bw^{\star}\|^2- \| \bw_{t+1}-\bw^{\star} \|^2) - \frac{\lambda}2\|\bw_t - \bw^{\star} \|^2 +2\eta_tC^2 \\
&& \le \sum_{i=1}^{T-1} \|\bw_{t+1} - \bw^{\star}\|^2(\frac{1}{2\eta_{t+1}} - \frac{1}{2\eta_{t}} - \frac{\lambda}2) + 2 C^2 \sum_{i=1}^\top \eta_t \\
&& \le \frac{2C^2 (1 + \ln T)}{\lambda}.
\enn
\end{proof}

\subsection*{Proof of Lemma \ref{lemma-bound-f}}

\begin{proof}
According to the update rule, we have
\begin{align}\label{lemma17}
\begin{split}
\|f_{t+1}\|_{\mathcal H} & = \|f_t - \eta_t \nabla L_t(f_t) \|_{\mathcal H} \\
&= \|(1-\lambda \eta_t) f_t + \frac{\eta_t}{\sqrt{A_t}} y_t\Psi_t k(.,\bx_t)- \frac{\eta_t}{2N_t\sqrt{A_t}} \sum_{i \in I_t}\left(2y_t\Psi_t + f_t(\bx_i)-\mu_{f_t}\right)k(.,\bx_i)\|_{\mathcal H} \\
& \le (1-\lambda\eta_t)\|f_t\|_{\mathcal H} +  \frac{\eta_t}{N_t\sqrt{A_t}} \Psi_t \| \sum_{i \in I_t} (k(.,\bx_t)-k(.,\bx_i))\|_{\mathcal H} \\ 
& +\frac{\eta_t}{2N_t\sqrt{A_t}} \|\sum_{i \in I_t}(f_t(\bx_i)-\mu_{t})k(.,\bx_i)\|_{\mathcal H}.
\end{split}
\end{align}
Because of the kernel is bounded, $k(\bx_t,\bx_i) \le k(\bx_t,\bx_t) \le 1$, we have
\begin{align}\label{lemma17-1}
\|(k(.,\bx_t)-k(.,\bx_i))\| = \sqrt{k(\bx_t,\bx_t)-2k(\bx_t,\bx_i)+k(\bx_i,\bx_i)} \le \sqrt 2.
\end{align}
According to \eqref{A} and \eqref{PsiMM}, we have 
\begin{align}\label{lemma17-2}
0 \le \frac{\Psi_t}{\sqrt{A_t}} = \frac{\ell_t+\sqrt{A_t}}{2\sqrt{A_t}} = \frac12 \left(1 + \frac{\ell_t}{\sqrt{\ell_t^2 + \sigma_{t}^2}}\right) \le 1.
\end{align}
Then we consider the third term of the inequality in \eqref{lemma17}, due to the triangle inequality and Cauchy-Schwarz inequality
\begin{align}\label{lemma17-3}
\begin{split}
\|\sum_{i \in I_t}\frac{(f_t(\bx_i)-\mu_{t})}{\sqrt{A_t}}k(.,\bx_i)\|_{\mathcal H}  &\le \sum_{i \in I_t} \| \frac{(f_t(\bx_i)-\mu_{t})}{\sqrt{A_t}}k(.,\bx_i)\| \le \sum_{i \in I_t} |\frac{(f_t(\bx_i)-\mu_{t})}{\sqrt{A_t}}| \\
&\le \sum_{i \in I_t} \frac 12 \left(1 + \left(\frac{f_t(\bx_i)-\mu_{t}}{\sqrt{A_t}} \right)^2 \right) \le N_t.
\end{split}
\end{align}
Pluging \eqref{lemma17-1},\eqref{lemma17-2},\eqref{lemma17-3} into \eqref{lemma17}, then we have
\ben
\|f_{t+1}\|_{\mathcal H}  &\le (1-\lambda \eta_t)\|f_t\|_{\mathcal H} + (\sqrt2+\frac12) \eta_t.
\enn
By expanding $\|f_t\|_{\mathcal H}$ iteratively and using $f_0=0$, we have
\ben
&&\|f_{t+1}\|_{\mathcal H} \le (\sqrt2+\frac12) (\eta_t + (1-\lambda \eta_t)\eta_{t-1}+\cdots +(1-\lambda\eta_t)(1-\lambda\eta_{t-1})\cdots(1-\lambda\eta_2)\eta_1 ) \\
&& = (\sqrt2+\frac12) \left(\frac{1}{\lambda t} + (1-\frac{1}{t})\frac{1}{\lambda(t-1)} + \cdots + (1-\frac{1}{t})(1-\frac{1}{t-1})\cdots(1-\frac{1}{2})\frac{1}{\lambda}   \right) \\
&& \le \frac{\sqrt2+\frac12}{\lambda}.
\enn
\end{proof}

\subsection*{Proof of Theorem \ref{The-regret-kernel-M}}
\begin{proof}
Let $f^{\star}$ be the best fixed classifier in hindsight. Since the maximum operation preserve convexity, then  $\psi_{M}$ which is derived by maximizing a sum of Hinge loss is also a convex function~\citep{Boyd2004}. By the $\lambda$-strong convexity of the loss function, $L_t(f)$, we have
\begin{align}\label{KRT}
\sum_{t=1}^T L_t(f_t) - \sum_{t=1}^T L_t(f^{\star}) \le \sum_{t=1}^T \left< \nabla L_t(f_t), f_t - f^{\star}\right>_{\mathcal H}-\frac \lambda 2 \| f_t - f^{\star}\|_{\mathcal H}^2,
\end{align}
Observing that
\ben
\|f_{t+1} - f^{\star}\|_{\mathcal H}^2 & = \|f_t -\eta_t \nabla L_t(f_t) - f^{\star}\|_{\mathcal H}^2= \|f_t - f^{\star}\|_{\mathcal H}^2 - 2\eta_t \left< \nabla L_t(f_t), f_t - f^{\star}\right>_{\mathcal H} + \eta_t^2 \|\nabla L_t(f_t)\|^2.
\enn
Then suming the above identity over $t \in [T]$ and rearranging, we have
\ben
\sum_{t=1}^T \left< \nabla L_t(f_t), f_t - f^{\star}\right>_{\mathcal H} &=&  \sum_{t=1}^T \frac{\eta_t}{2}\|\nabla L_t(f_t)\|_{\mathcal H}^2 +  \sum_{t=1}^T \frac{1}{2\eta_t}(\|f_t - f^{\star}\|_{\mathcal H}^2 - \|f_{t+1} - f^{\star}\|_{\mathcal H}^2).
\enn
Combine with \eqref{KRT}, we have
\ben
Regret_T^{AUC}  && \le \sum_{t=1}^T \frac{\eta_t}{2}\|\nabla L_t(f_t)\|_{\mathcal H}^2 +\sum_{t=1}^T \frac{1}{2\eta_t}(\|f_t - f^{\star}\|_{\mathcal H}^2 - \|f_{t+1} - f^{\star}\|_{\mathcal H}^2 )-\frac \lambda 2 \| f_t - f^{\star}\|_{\mathcal H}^2 \\
&& \le \sum_{t=1}^T \frac{\eta_t}{2}\|\nabla L_t(f_t)\|_{\mathcal H}^2 + \sum_{t=1}^{T-1} (\frac{1}{2\eta_{t+1}} - \frac{1}{2\eta_{t}} -\frac \lambda 2 )\|f_{t+1} - f^{\star}\|_{\mathcal H}^2  \\
&&= \sum_{t=1}^T \frac{\eta_t}{2}\|\nabla L_t(f_t)\|_{\mathcal H}^2.
\enn
According to \eqref{lemma17-1},\eqref{lemma17-2},\eqref{lemma17-3}, we can bound $\|\nabla L_t(f_t)\|_{\mathcal H}$  as
\ben
&&\|\nabla L_t(f_t)\|_{\mathcal H}  = \| \lambda f_{t} - \frac{1}{\sqrt{A_t}} y_t\Psi_t k(.,\bx_t) + \frac{1}{2N_t\sqrt{A_t}} \sum_{i \in I_t}(2y_t\Psi_t + f_t(\bx_i)-\mu_{t})k(.,\bx_i)  \|_{\mathcal H}  \\
&&\le \lambda \|f_t\|_{\mathcal H} +  \frac{1}{N_t\sqrt{A_t}} \Psi_t \| \sum_{i \in I_t} (k(.,\bx_t)-k(.,\bx_i))\|_{\mathcal H} +\frac{1}{2N_t\sqrt{A_t}} \|\sum_{i \in I_t}(f_t(\bx_i)-\mu_{t})k(.,\bx_i)\|_{\mathcal H} \\
&&= \lambda \|f_t\|_{\mathcal H} + (\sqrt2+\frac12) \\
&&\le (2\sqrt2+1).
\enn
Then we have
\ben
Regret_T^{AUC}  &\le \sum_{t=1}^T \frac{\eta_t}{2}\|\nabla L_t(f_t)\|_{\mathcal H}^2 \le \frac{(2\sqrt2+1)^2}{2\lambda} \sum_{t=1}^T \frac{1}{t} \le \frac{(2\sqrt2+1)^2}{2\lambda}(1+\ln T).
\enn
\end{proof}

\vskip 0.2in
\bibliography{luo-online-auc-v1}

\end{document}